\def\year{2022}\relax
\documentclass[letterpaper]{article} % DO NOT CHANGE THIS
\usepackage{aaai22}  % DO NOT CHANGE THIS
\usepackage{times}  % DO NOT CHANGE THIS
\usepackage{helvet} % DO NOT CHANGE THIS
\usepackage{courier}  % DO NOT CHANGE THIS
\usepackage[hyphens]{url}  % DO NOT CHANGE THIS
\usepackage{graphicx} % DO NOT CHANGE THIS
\usepackage{natbib}  % DO NOT CHANGE THIS AND DO NOT ADD ANY OPTIONS TO IT
\usepackage{caption} % DO NOT CHANGE THIS AND DO NOT ADD ANY OPTIONS TO IT
\DeclareCaptionStyle{ruled}{labelfont=normalfont,labelsep=colon,strut=off} % DO NOT CHANGE THIS
\usepackage{amsmath}
\usepackage{amssymb}
\usepackage{epsfig, psfrag}
\usepackage{amsthm}
\usepackage{latexsym}
\usepackage{bbm}
\usepackage{textcomp}
\usepackage[ruled,noend,noline,linesnumbered]{algorithm2e}
\SetKw{Init}{Initialize:}
\SetKw{Def}{Definition:}

\usepackage[bottom]{footmisc}

\usepackage{comment}

\usepackage{mathtools}
\mathtoolsset{showonlyrefs=true}

\usepackage{epstopdf}
\usepackage{subcaption}
\usepackage{xcolor}

\usepackage{enumitem}

\newtheorem{theorem}{Theorem}
\newtheorem{lemma}{Lemma}

\newtheorem{fact}{Fact}
\theoremstyle{definition}

\renewenvironment{quote}
  {\list{}{\rightmargin=0cm \leftmargin=0.25cm}%
   \item\relax}
  {\endlist}

\def\expe{\mathbb{E}}   % expectation
\def\P{\mathsf{P}}   % probability

\def\argmax{\mathop{\rm arg\,max}}

\def\reg{\mathsf{R}}

\def\mc{\mathcal}

\newcommand{\indicate}[1]{\mathbf{1}{\{#1\}}}

\title{Partner-Aware Algorithms in Decentralized Cooperative Bandit Teams}

\author {
    Erdem B\i y\i k\textsuperscript{\rm 1},
    Anusha Lalitha\textsuperscript{\rm 1},
    Rajarshi Saha\textsuperscript{\rm 1},
    Andrea Goldsmith\textsuperscript{\rm 1,2},
    Dorsa Sadigh\textsuperscript{\rm 1,3}\\
}
\affiliations {
    \textsuperscript{\rm 1} Department of Electrical Engineering, Stanford University \\
    \textsuperscript{\rm 2} Department of Electrical and Computer Engineering, Princeton University \\
    \textsuperscript{\rm 3} Department of Computer Science, Stanford University \\
    \{ebiyik, alalitha, rajsaha, dorsa\}@stanford.edu, goldsmith@princeton.edu
}

\begin{document}

\maketitle

\begin{abstract}
When humans collaborate with each other, they often make decisions by observing others and considering the consequences that their actions may have on the entire team, instead of greedily doing what is best for just themselves. We would like our AI agents to effectively collaborate in a similar way by capturing a model of their partners. In this work, we propose and analyze a decentralized Multi-Armed Bandit (MAB) problem with coupled rewards as an abstraction of more general multi-agent collaboration. We demonstrate that na\"ive extensions of single-agent optimal MAB algorithms fail when applied for decentralized bandit teams. Instead, we propose a \emph{Partner-Aware} strategy for joint sequential decision-making that extends the well-known single-agent Upper Confidence Bound algorithm. We analytically show that our proposed strategy achieves logarithmic regret, and provide extensive experiments involving human-AI and human-robot collaboration to validate our theoretical findings. Our results show that the proposed partner-aware strategy outperforms other known methods, and our human subject studies suggest humans prefer to collaborate with AI agents implementing our partner-aware strategy.
\end{abstract}

\section{Introduction}
\label{sec:intro}
One of the key characteristics of human-human interaction is people's ability to seamlessly anticipate and take complementary actions when working with others.
For example, the moment our partner reaches for a box of cereal, we automatically walk to the fridge to grab milk.
The success of multi-agent systems or human-AI teams usually depends on not only each agent's actions, but also how well they model other agents' policies and the interplay between them.
As another example, consider a semi-autonomous car where both the actions of the driver, e.g., keeping or changing lanes, and assistive guidance, e.g., corrections that keep the car inside the lanes, determine the control of the vehicle. Here, we expect the guidance to predict the driver's intent and augment their actions to enhance safety and comfort. 

\begin{comment}
As another example, consider a human-robot team tasked with stacking burgers in a fast-food restaurant, where they take turns stacking the ingredients (see Fig.~\ref{fig:frontfig}). Suppose the human is responsible for the patty and the cheese in the burger, whereas the robot stacks tomatoes and lettuce. As many people have strong opinions about in what order these ingredients should be stacked \cite{burger_blog}, the robot needs to predict the human partner's actions to better coordinate with them on stacking the burger and satisfy the customers.
\end{comment}

Decentralized learning is particularly challenging when some agents have limited information of the outcomes of the actions taken by the team. 
In the car example, even though both human and assistive guidance share the same goal of safety and comfort, the guidance may not fully observe the driver's internal objective of, for instance, changing lanes to exit and get gas at a cheaper station.
%A similar problem may also arise in the burger stacking example if the robot has only partial information about whether or not a guest liked a burger. Therefore, the robot might take suboptimal actions even though the human may have already discovered the optimal action and expected the robot to comply. This lack of successful coordination often causes significant losses in the task objectives.
Although explicit communication can alleviate some of these challenges, it is often impractical or expensive: we cannot expect the guidance system to ask for and expect feedback after every decision of the driver. On the other hand, humans rely on implicit communication for coordination in many interactive settings~\cite{breazeal2005effects,che2020efficient,losey2020learning}. They generally make effective inferences by simply observing and reasoning over their partner's actions. Hence, we question if AI agents can accurately model others' policies without explicit communication to effectively coordinate and cooperate.

% Not surprisingly, 
Previous works such as theory of mind~\cite{simon1995information,baker2017rational,brooks2019building,lee2019theoryofmind} and opponent modeling in multi-agent learning~\cite{foerster2018learning,shih2021critical,xie2020learning}
showed the performance of human-AI and multi-agent teams may significantly increase if the agents accurately model each other's policies.
However, most of these approaches require recursive belief modeling or rely on learned partner representations, which can often be complex and computationally intractable.

Our goal is to develop a simple and tractable approach for modeling partners in decentralized multi-agent teams that is guaranteed to improve performance.
% To this end, we study the problem of decentralized multi-agent collaborative learning and the importance of partner-modeling.
We specifically focus on decentralized Multi-Armed Bandit (MAB) problems, which extend the stochastic MAB, a fundamental model for sequential decision-making to explore an agent's environment efficiently. 
Our decentralized MAB formulation captures the essential elements of multi-agent collaborative learning.
First, we model the team reward to be dictated by the actions of all agents, which is common in many realistic collaborations, e.g., the safety and comfort of a semi-autonomous vehicle depend on both the driver's and the guidance system's actions.
Second, we model the heterogeneity in the information available to each agent by introducing partial observability over rewards, e.g., the vehicle does not always accurately observe whether the human is looking for the fastest route or the cheapest gas station.
Hence, our formulation requires collaboration among agents to accomplish the task of learning the optimal team action while only observing each others' actions. 

One might hope that na\"ive extensions of well-known bandit algorithms such as \emph{Thompson Sampling} and \emph{Upper Confidence Bound (UCB)} would be sufficient for effective collaboration in these settings. However, we demonstrate these extensions fail to provide logarithmic regret. 
% We demonstrate na\"ive extensions of well-known bandit algorithms such as \emph{Thompson Sampling} and \emph{Upper Confidence Bound (UCB)} to this decentralized formulation fail to provide logarithmic regret. 
% However, applying belief modeling techniques described earlier to model the partner's actions can get computationally intractable and usually do not scale with the number of agents. 
Our key insight is to leverage the simplicity of these well-known algorithms while predicting our partner's actions --- \emph{make the agent with lower observability of rewards follow the agent with the higher observability.} Specifically, we propose a computationally simple partner-aware bandit learning algorithm where the follower learns to predict its partner’s actions while choosing its own action.
%This enables the agents to collaborate when there is partial observability of rewards and when their rewards are coupled.
We analytically show that this algorithm incurs regret logarithmic in time horizon.

Our main contributions are:
\begin{itemize}[leftmargin=20pt]
    \item We propose a computationally efficient partner-aware bandit algorithm, which anticipates the partner's action and effectively coordinates with the partner.
    
    \item We analytically prove our proposed algorithm significantly improves the team performance and provides logarithmic regret.
    
    \item Finally, we conduct extensive simulations and an in-lab collaborative robot experiment shown in Fig.~\ref{fig:frontfig}. Our results suggest our algorithm significantly improves the team performance and is preferred by the users.
\end{itemize}

\begin{figure*}
  \includegraphics[width=\textwidth]{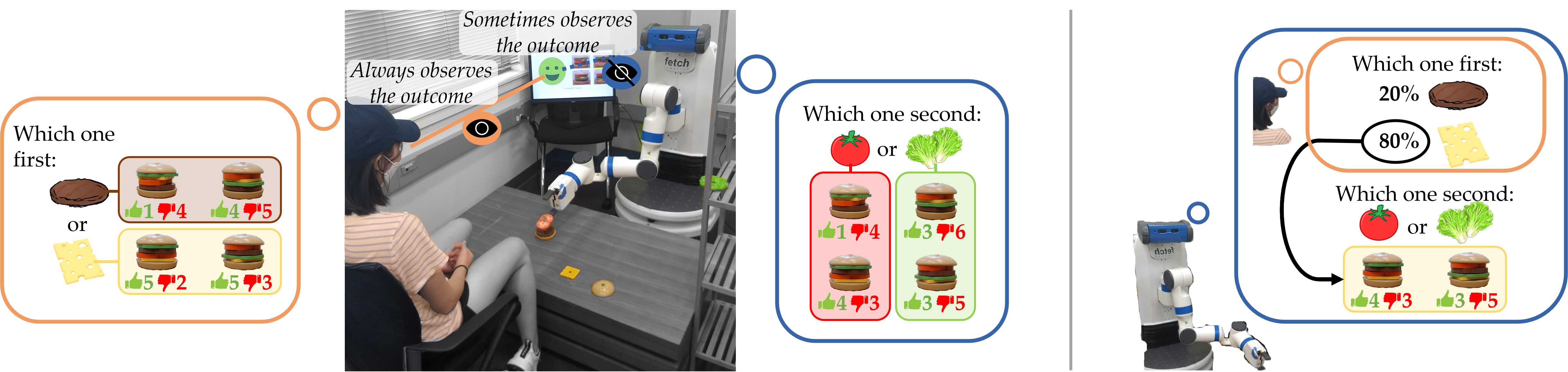}
  \caption{AI agents collaborating with humans should model their human partners. (Left) A robot and a human collaborate on a burger stacking task. (Right) The robot decides its actions by modeling the actions of the human.}
  \label{fig:frontfig}
\end{figure*}

\section{Related Work}
\noindent\textbf{Multi-Agent Multi-Armed Bandits.}
Existing decentralized cooperative MAB algorithms make one or more of the following assumptions: (i) agents independently interact with the same MAB \cite{DBLP:conf/aaai/LupuDP19}, (ii) they use sophisticated communication protocols to exchange information about rewards and the number of times actions were played \cite{landgren_cdc16,martinez_nips2019,sankararaman2019social,shahrampour_dece_MAB,barrett2014communicating}, (iii) when sophisticated communication is not possible, agents share their latest action and reward \cite{madhushani2019heterogeneous}.

These assumptions are often required to simplify the analysis, but are not realistic in most human-AI interactions, e.g., collaborative transport, assembly, cooking, or autonomous driving, where (i) agents’ actions influence the outcome for the whole team, (ii) they do not have explicit communication channels or might have different state, action representations that are difficult to communicate, or (iii) they have different capabilities, e.g., noisier sensors.

In our work, we do not make any of these assumptions. We advance the current literature by analyzing a more realistic model suited for collaborative human-AI interaction, where we: (i) relax the assumption of independent agents through coupled rewards, (ii) allow only implicit communication, i.e., agents can only observe each other’s actions and not the rewards, (iii) relax the assumption on homogeneity of agents, i.e., some agents may receive noisier rewards. Hence, existing algorithms in multi-agent multi-armed bandits are not applicable in our setting. Our contribution is a novel and computationally simple partner-aware algorithm for decentralized collaboration, and proving it incurs logarithmic regret for any finite number of arms.

\smallskip
\noindent\textbf{Multi-Agent Learning.} 
% Multi-agent interaction often requires multiple agents to take actions in a shared environment. 
Recent works have shown the importance of partner modeling in multi-agent environments \cite{devin2016implemented,zhu2020multi}. \citet{foerster2018learning} proposed an algorithm that 
% takes the learning of other agents' into account and
improves performance in repeated prisoner's dilemma using opponent modeling. 
% \cite{zhu2020multi} studied multi-agent planning in safety-critical settings and concluded agents more accurately predicting the behavior of their partners improve the safety and performance of the team.
% studied implicit communication where each robot has only partial information about the environment, and 
\citet{losey2020learning} showed that agents can implicitly communicate through their actions. Other works have learned partner representations for effective coordination~\cite{shih2021critical,xie2020learning,grover2018learning}.
%and convey information by alternating roles.
However, these approaches are either not guaranteed to effectively coordinate as they heavily rely on learned representations or can lead to suboptimal solutions in multi-agent MAB, where agents need to take the optimal action more frequently over time to avoid linearly growing regrets.

\smallskip
\noindent\textbf{Humans in Multi-Armed Bandits.} 
\citet{zhang2013cheap} compared how various algorithms match with the actions of humans playing a stochastic MAB.
% In this work, we show the importance of partner-modeling in decentralized MAB problems. 
While our algorithm does not specifically model the partner as an agent incorporating the imperfections humans have, we observe via our user studies that it can collaborate well not only in multi-AI teams but also with human partners.

%\section{Decentralized Teams}
%\input{sec/decentralized_teams}

\section{Problem Setting}
\label{sec:problem}
In this section, we present a decentralized MAB formulation that captures essential aspects of multi-agent decentralized collaborative learning. 

\smallskip
\noindent\textbf{Running Example.} Consider a human-robot team tasked with stacking burgers in a fast-food restaurant, where they stack the ingredients together (see Fig.~\ref{fig:frontfig}). Suppose the human is responsible for the patty and the cheese in the burger, whereas the robot stacks tomatoes and lettuce. As many people have strong opinions about in what order these ingredients should be stacked \cite{burger_blog}, the robot should predict the human's actions to better coordinate on stacking the burger.
If the robot only has partial information about whether a guest liked a burger, it might take suboptimal actions even though the human may have already discovered the optimal action and expected the robot to comply.

Formally, at every time instant $t$ each agent $i$, where $i \in \{1,2\}$,\footnote{We generalize to more agents in Section~\ref{sec:sim}.} chooses an action $a^{(i)}_t \in \mc{A}_i$ locally. The team action is defined as the union of both agents' actions, i.e., $a_t := (a^{(1)}_t, a^{(2)}_t)$. The team action space is denoted by  $\mc{A} = \mc{A}_1 \times \mc{A}_2$. It is helpful to think of $\mc{A}$ as a team action matrix and each possible team action as a cell of the matrix. Thus, at any time instant $t$, agent 1 selects one row out of the $|\mc{A}_1|$ rows and agent 2 selects one column out of the $|\mc{A}_2|$ columns. We assume the agents select their local actions simultaneously, i.e., before observing their partner's current action.

For each team action $a \in \mc{A}$, the rewards $\{r^{\ast}_t(a)\}_{t\geq 1}$ are sampled independently from a Bernoulli distribution with unknown mean $\mu_a \in [0,1]$. Note the reward is a function of both agents' actions. We refer to such scenarios as settings with \emph{coupled rewards}, where the actions of all agents govern the reward received by each agent. This necessitates that each agent learns to account for others' actions instead of greedily optimizing its own rewards. In the burger example, the robot learns to choose actions to stack the ingredients in the right order while learning to predict the human's action. 
 
We assume the agents observe the rewards with a fixed probability $p_i$, i.e., $r^{(i)}_t(a) = r^{\ast}_t(a)$ with probability $p_i$, and the agent incorrectly assumes $r^{(i)}_t(a) = 0$ with probability $1-p_i$.\footnote{Agents do not know when they failed to observe the reward. Knowing it is a simpler setting, where agents update their local reward statistics only when a reward is observed.} 
%Put another way, the agents may not always observe the rewards and when an agent fails to observe the reward, it incorrectly assumes the team's action did not generate a reward.
We refer to such scenarios as settings with \emph{partial reward observability}. Going back to our burger stacking example, humans may have better reward observability---they could better sense when a guest has been happy about the burger---whereas the robot needs an explicit feedback.

We assume agents observe each other's actions but not the rewards: at any time $t$ they have the knowledge of all past team actions $\{a_{\tau}\}_{\tau =1}^{t-1}$ and only their own local rewards. 

To summarize, our setting considers truly realistic coordination scenarios where the agent's actions influence the outcome for the whole team. In addition, the agents only observe each other's actions and do not have access to direct communication channels, which covers the difficult case where agents are heterogeneous and might have different modes of communication---a human can easily use language, but that might not be as easy for an AI agent to interpret or use. Finally, as most realistic teams, we assume agents have different capabilities (e.g. sensing capabilities) leading to different observations corresponding to their own local reward observations.

The goal of the team at every time step $t$ is to select an action $a_t$ that maximizes the average team reward $r_t(a_t) \!=\! \frac{1}{2}r^{(1)}_t(a_t) + \frac{1}{2}r^{(2)}_t(a_t)$.\footnote{Since the rewards are coupled, our analysis and Theorem~\ref{thm:main_theorem} will extend to the case where the team reward is any linear combination of agents' rewards.} Let $a_{\ast}$ denote the optimal team action, i.e., $a_{\ast} \!=\! \argmax_{a \in \mc{A}} \expe[r_t(a)]$. Thus, agent 1 seeks to identify the \textit{optimal row}  $(a_{\ast}^{(1)}, \cdot)$ and agent 2 tries to identify the \textit{optimal column} $(\cdot, a_{\ast}^{(2)})$. Their intersection is the optimal cell $a_{\ast}\! =\! (a_{\ast}^{(1)}, a_{\ast}^{(2)})$. Alternatively, the team aims to minimize the cumulative regret defined as $\reg(T):= \expe\left[\sum_{t=1}^T \left(r_t(a_{\ast}) - r_t(a_t)\right) \right]$.
In this work, we aim to design decentralized team action strategies for each agent $i$ that select $a^{(i)}_t$ as a function of past local rewards $r^{(i)}_1(a_1), \ldots, r^{(i)}_{t-1}(a_{t-1})$ and the team actions $a_1, \ldots, a_{t-1}$. We are interested in cases where at least one agent has partial reward observability, i.e, $p_i \neq 1$ for some $i \in \{1,2\}$.

\section{Partner-Aware Bandit Learning}

We now present a learning algorithm for the decentralized collaborative MAB when the agents have partial reward observability, and their rewards are coupled.
% Because the agents have different partial observability of rewards, the sequence of rewards observed by each agent up to any given time instant may be different.
Different local observations due to partial reward observability can lead to the agents wanting to select different team actions. Since the agents' rewards are coupled, such a mismatch in the agents' action-choices can cause them to explore their action space inefficiently as a team. To successfully collaborate, the agents need to learn to predict their partners' actions correctly. Modeling partner's belief states and action-strategy has been well-studied in the theory of mind literature; however, such recursive belief modeling techniques can get computationally prohibitive and do not scale well with the number of agents \cite{hellstrom2018understandable}. Instead, we introduce a computationally simple way of predicting the partner's actions in the collaborative multi-armed bandit domain---which is a useful abstraction that enables theoretically analyzing multi-agent interactions. The core ideas of our approach, though simple, lead to an analytical algorithm with logarithmic regret, and can provide insight for partner modeling beyond multi-armed bandits.

Let $p_{\text{max}} := \max\{p_1, p_2\}$ and $p_{\text{min}} := \min\{p_1, p_2\}$. We refer to the agent with higher reward observability ($p_{\text{max}}$) as the \textit{leader} and the other agent as the  \textit{follower}.\footnote{Our algorithm extends to the case where $p_1$ and $p_2$ are unknown, in which case leader and follower roles are assigned randomly and the $p_{\text{max}}$ terms in the denominators of Theorem~\ref{thm:main_theorem} will be replaced with $p_{\text{min}}$.} In our approach, the follower learns to predict the leader's actions. It chooses its local action assuming the leader's current action will match its prediction. As its predictions become more accurate, the leader leads the follower to explore the optimal row in the action matrix $\mc{A}$. Since the leader has higher reward observability, the team can efficiently explore the action matrix.
We rewrite the team action based on leader and follower assignment as:
% Assigning the agents with a leader or follower role allows us to rewrite the team action as 
$a_t \!=\! (a^{(L)}_t, a^{(F)}_t) \!\in\! \mc{A}\!:=\! \mc{A}_L\!\times\mc{A}_F$. We denote the optimal team action as $a_{\ast} = (a^{(L)}_{\ast}, a^{(F)}_{\ast})$. Similarly, $r^{(L)}_t$ and $r^{(F)}_t$ denote the observed rewards. %Also, let $n_{a}(t)$ denote the number of times team action $a \in \mc{A}$ was taken by the team up to time $t$.

\smallskip
\noindent \textbf{Partner-Aware UCB: Follower.} 
% Next, we present our partner-aware UCB algorithm for the follower. 
We provide the pseudocode in Algorithm~\ref{alg:PA_UCB_follower}. At every time step $t$, the follower predicts the leader's current action by sampling from a distribution $\tilde{\rho}^{(L)}_t$ over leader's action space $\mc{A}_L$ (line 5), which is obtained by normalizing the histogram computed from the leader's past $W$ actions. Intuitively, $\tilde{\rho}^{(L)}_t$ serves an approximation of the leader's action selection strategy. As the leader becomes more confident about the optimal action and starts to exploit, the distribution $\tilde{\rho}^{(L)}_t$ concentrates over the optimal action. Hence, the follower gets more accurate in its predictions of the leader's actions. At every time step, the follower uses its prediction of leader's action $\tilde{a}^{(L)}_t$ to fix a row in the action matrix $\mc{A}$ and choose one of the $|\mc{A}_F|$ columns. To do so, it computes an upper confidence bound on the mean value for the actions in the row $\tilde{a}^{(L)}_t$, and chooses the action maximizing the upper confidence bound (line 6):
\begin{align*}
    a^{(F)}_t := \argmax_{a^{(F)} \in \mc{A}_F} \hat{\mu}^{(F)}_{(\tilde{a}^{(L)}_t,a^{(F)})}\!(t\!-\!1) + \sqrt{\frac{c^{(F)} \log 1/\delta}{n_{(\tilde{a}^{(L)}_t,a^{(F)})}\!(t\!-\!1)}},
\end{align*}
where $\hat{\mu}^{(F)}_a$ denotes the empirical mean of the follower's local rewards, $n_a$ denotes the action count for any team action $a$, and $c^{(F)},\delta>0$ are exploration parameters.

In short, the follower predicts the leader's action by looking at its past $W$ actions. If the leader takes some actions more frequently, then the follower predicts those actions with high probability and aids the leader in exploring them.

\newlength{\textfloatsepsave} \setlength{\textfloatsepsave}{\textfloatsep} \setlength{\textfloatsep}{2pt}% Reduce the margin after the algorithm block -- we will bring it back later when the block ends
\SetInd{0.7em}{0em}
\begin{algorithm}[t]
    %\SetAlgoLined
    \DontPrintSemicolon
    \KwIn{$\delta\!>\!0$, $W\!\geq\!1$, exploration constant: $c^{(F)}\!>\!0$}
    \Def Denote empirical mean $\hat{\mu}^{(F)}_a(t) = \frac{\sum_{\tau = 1}^t r^{(F)}_{\tau}(a_{\tau}) \indicate{a_{\tau} = a}}{n_a(t)} \quad \forall\, a \in \mc{A}$\;
    Denote upper confidence bound  $f^{(F)}_a(t,\delta) = \hat{\mu}^{(F)}_a(t-1) + \sqrt{\frac{c^{(F)} \log 1/\delta}{n_a(t-1)}} \quad \forall\, a \in \mc{A}$\;
    \Init $n_a(0) = 0, \hat{\mu}^{(F)}_a(0) = 0$, $f^{(F)}_a(1,\delta) = \infty$ for all $a \in \mc{A}$, set $\rho^{(L)}_t(a) = \frac{1}{|\mc{A}_L|} \, \forall a\in \mc{A}_L$ \;
    \For{$t = 1, \ldots, T$}{
        Predict leader's action by sampling $\tilde{a}^{(L)}_{t} \sim \rho^{(L)}_t$\;
        Select $ a^{(F)}_t \gets \argmax_{a \in \mc{A}_F} f^{(F)}_{(\tilde{a}^{(L)}_t, a)}(t,\delta)$ \;
        Perform $a^{(F)}_t$\;
        Observe partner's action $a^{(L)}_t$ and reward $r^{(F)}_t(a_t)$\;
        Update $n_{a_t}(t) \!\gets\! n_{a_t}(t-1) + 1 $ \; 
        Update $\hat{\mu}^{(F)}_{a_t}(t)$ and $f^{(F)}_{a_t}(t,\delta)$\;
        Update $\rho^{(L)}_{t+1}\!(a)\!\gets\! \frac{\sum_{\tau \!=\! \max\{1, t\!-\!W\!+\!1\}}^t \!\indicate{a^{(L)}_{\tau} \!=\! a}}{\min\{t,W\}} \: \forall a \!\in\! \mc{A}_L$ 
    }
\caption{Partner-Aware UCB: Follower}
\label{alg:PA_UCB_follower}
\end{algorithm}
% As promised before the algorithm block, bring the old textfloatsep back
\setlength{\textfloatsep}{\textfloatsepsave}

\smallskip
\noindent \textbf{Partner-Aware UCB: Leader.} Now we present our partner-aware UCB algorithm for the leader. This algorithm is an extension of the well-known UCB algorithm. We provide the pseudocode in Algorithm~\ref{alg:PA_UCB_leader}. For each team action, the leader computes an upper confidence bound on its mean value using the local observations. The leader then selects a team action that maximizes the upper confidence bound (line 6), similar to the follower's selection criterion. The leader then plays its own coordinate of the team action it selected, and it repeats every action it selects for $L$ consecutive time steps (line 8).

%\newlength{\textfloatsepsave}
\setlength{\textfloatsepsave}{\textfloatsep} \setlength{\textfloatsep}{2pt}% Reduce the margin after the algorithm block -- we will bring it back later when the block ends
\begin{algorithm}[t]
    %\SetAlgoLined
    \DontPrintSemicolon
    \KwIn{$\delta >0$, $L \geq 1$, exploration constant: $c^{(L)} > 0$}
    \Def Denote empirical mean $\hat{\mu}^{(L)}_a(t) = \frac{\sum_{\tau = 1}^t r^{(L)}_{\tau}(a_{\tau}) \indicate{a_{\tau} = a}}{n_a(t)} \quad \forall\, a \in \mc{A}$\;
    Denote upper confidence bound $f^{(L)}_a(t,\delta) = \hat{\mu}^{(L)}_a(t-1) + \sqrt{\frac{c^{(L)} \log 1/\delta}{n_a(t-1)}} \quad \forall\, a \in \mc{A}$\;
    \Init $n_a(0) = 0, \hat{\mu}^{(L)}_a(0) = 0$, $f^{(L)}_a(1,\delta) = \infty$ for all $a \in \mc{A}$ \;
    \For{$t = 1, \ldots, T$}{
        \eIf{$t\bmod L = 1$ }{
            Select $\left(a^{(L)}_t, \cdot \right) \gets \argmax_{a \in \mc{A}} f^{(L)}_a(t,\delta)$ \;
            }{
            $a^{(L)}_t\gets a^{(L)}_{t-1}$\;
        }
        Perform $a^{(L)}_t$\;
        Observe partner's action $a^{(F)}_t$ and reward $r^{(L)}_t(a_t)$\;
        Update $n_{a_t}(t) \gets n_{a_t}(t-1) + 1 $ \; 
        Update $\hat{\mu}^{(L)}_{a_t}(t)$ and $f^{(L)}_{a_t}(t,\delta)$
    }
\caption{Partner-Aware UCB: Leader}
\label{alg:PA_UCB_leader}
\end{algorithm}
% As promised before the algorithm block, bring the old textfloatsep back
\setlength{\textfloatsep}{\textfloatsepsave}

As the follower predicts the leader's action based on each action's frequency in the past $W$ time steps, the leader repeating its actions more than once ($L\!>\!1$) ensures the follower's prediction matches the leader's action with a high probability.  We use this for our analysis, but employ $L\!=\!1$ in practice to avoid potential losses due to repetitive actions.

\begin{theorem}\label{thm:main_theorem}
For any horizon $T$, if $\delta = \frac{1}{T^2}$, $L = 2$ and $W = 1$, the cumulative regret of partner-aware bandit learning algorithm, as defined in Algorithms~\ref{alg:PA_UCB_follower} and~\ref{alg:PA_UCB_leader}, is logarithmic in the horizon $T$. Specifically, the cumulative regret $\reg(T)$ can be upper bounded by
\begin{align}
    (p_{\text{max}}+p_{\text{min}})&\Delta_{\text{max}}\Bigg[\sum_{i \neq a^{(L)}_{\ast}} \frac{16}{p^2_{\text{max}}\Delta^2_{(i, j^{\ast}(i))}} \log T + \nonumber\\
    &\sum_{i \in \mc{A}_L} \sum_{j \neq j^{\ast}(i)}\frac{16}{p^2_{\text{max}}\tilde{\Delta}^2_{(i, j)}}  \log T + \frac{3\lvert\mc{A}_L\rvert\lvert\mc{A}_F\rvert}{2}\Bigg],
\end{align}
where $\Delta_{\text{max}} = \max_{a \in \mc{A}} \Delta_a$,  $j^{\ast}(i) \!:=\! \argmax_{j \in \mc{A}_F} \mu_{(i,j)}$, %denotes optimal column with highest payoff in row $i$,
$\tilde{\Delta}_{(i,j)} \!=\! \mu_{(i,j^{\ast}(i))}\!-\! \mu_{(i,j)}$ for $i \!\in\! \mc{A}_L$ and $j \!\neq\! j^{\ast}(i)$, and $\Delta_{(i,j^{\ast}(i))} \!=\! \mu_{a^{\ast}}\!-\! \mu_{(i,j^{\ast}(i))}$ for $i \!\neq\! a^{(L)}_{\ast}$.
\end{theorem}

This theorem analyzes a special case, $L\!=\!2$ and $W\!=\!1$, where the follower predicts the leader will take the same action it took in the last time step. As the leader repeats its actions twice, these predictions are correct for at least half of all time steps.\footnote{Similarly, the proof can be generalized to any $\lfloor\! W/2 \!+ \!1\!\rfloor \!<\! L$.} Thus, the agents jointly explore the row which the leader intends to explore for at least half of the time steps. After the leader converges to the optimal row $a^{(L)}_{\ast}$, the leader and follower jointly explore the optimal column to learn the optimal cell $(a^{(L)}_{\ast}, a^{(F)}_{\ast})$. Proof of Theorem~\ref{thm:main_theorem} uses this intuition (see Appendix~\ref{appendix:main_proof}).

\section{Simulations}
\label{sec:sim}
We now assess the performance of our Partner-Aware UCB algorithm through a set of simulations.\footnote{Code at: \url{https://sites.google.com/view/partner-aware-ucb}} Unless otherwise noted, we set $|\mc{A}_1|=|\mc{A}_2|=2$, $p_1=1$, $p_2=0.5$, $c^{(L)}=c^{(F)}=0.025$ in these simulations.
%In particular, we first validate our theoretical results (Sec.~\ref{subsec:sim0}) and compare our algorithm to na\"ive extensions of Thompson Sampling and UCB
%(Sec.~\ref{subsec:sim1}).
%We also investigate the performance of Partner-Aware UCB under varying observabilities and action spaces (Sec.~\ref{subsec:sim2}). Finally, we propose an extension to settings that involve more than two agents, and empirically show the sublinear regret of this extension (Sec.~\ref{appendix:sim3}).

\smallskip
\noindent\textbf{Validation of Theoretical Results.} \label{subsec:sim0}
We start with validating the theoretical results we established in Theorem~\ref{thm:main_theorem}.
For this, we ran a simulation with fixed reward means.\footnote{We set $\mc{A}_1=\mc{A}_2=\{0,1\}$ and $\mu_{(0,0)}=0.8$, $\mu_{(0,1)}=0.4$, $\mu_{(1,0)}=0.2$, $\mu_{(1,1)}=0.6$. This is a difficult setting, as agents may easily converge to the local optimum $a=(1,1)$.}

Figure~\ref{fig:sim0}~(left) shows the results that validate the theorem. It also provides a comparison between different $L$ values. Having seen that the algorithm performs comparably with no significant difference with varying $L$, we use $L=1$ for the rest of the simulations and experiments, because it reduces the leader to a standard UCB agent and relaxes the assumption that the leader repeats its actions, which is particularly desirable in human-AI interaction with the human acting in the leader role. 

\begin{figure}[h]
    \centering
    \includegraphics[width=\columnwidth]{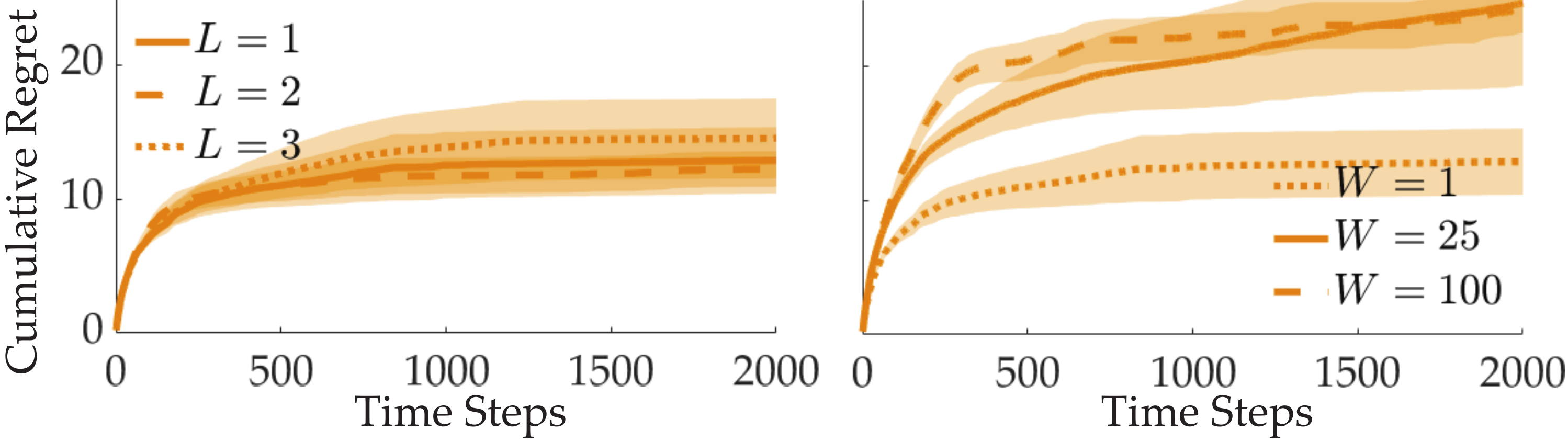}
    \caption{Cumulative regret values over $100$ runs for varying (left) $L$ and (right) $W$. Shaded regions show standard error.}
    \label{fig:sim0}
\end{figure}

We compare different $W$ in Fig.~\ref{fig:sim0} (right). Here, $W\!=\!1$ outperforms larger window widths. However, this assumes the follower is paired with a UCB leader, who selects the next action based on the entire history of actions and local rewards. This is unrealistic when interacting with a human leader. Humans are often bounded rational and make decisions only based on the most recent information~\cite{zhang2013cheap,simon1995information}.  We thus will use higher values of $W$ in practice to increase the follower's horizon to the past, which can potentially improve robustness (see Appendix~\ref{appendix:ucb_matched_with_kg}). 
% In fact, when Partner-Aware UCB is matched with an agent who follows knowledge gradient, larger $W$ performs better (see Appendix~\ref{appendix:ucb_matched_with_kg}). Thus, we vary $W$ depending on the problem horizon in the rest of the paper.

% Moreover, we compare different values of $W$ in Fig.~\ref{fig:sim0} ($L=1$ and keeping the other parameters the same). We observe $W=1$ outperforms larger window widths. However, this is when the leader also employs our algorithm, which selects the next action as a function of the entire history of actions and local rewards. However, humans do not necessarily employ this algorithm. They are known to be more myopic and as some previous works suggest, their decision making can be better captured by algorithms such as knowledge gradient~\cite{zhang2013cheap}. This myopia causes their next action to be a function of only the recent information. To alleviate such suboptimalities of human leaders, one may use higher values of $W$ to increase the horizon to the past, which could potentially improve robustness. In fact, when Partner-Aware UCB is matched with an agent who follows knowledge gradient, larger $W$ performs better (see Appendix~\ref{appendix:ucb_matched_with_kg}). Thus, we vary $W$ depending on the problem horizon in the rest of the paper.

\smallskip
\noindent\textbf{Effect of Partner-Awareness.} \label{subsec:sim1}
When established MAB algorithms, such as UCB and Thompson sampling, are na\"ively used in the multi-agent case, each agent attempts to solve a single-agent MAB problem in the team action space. While they are known to produce logarithmic regret in the single-agent case, the collaborative problem is much more challenging, since agents can only decide their part of the team action. Hence, we empirically show that simply pairing such standard algorithms in the multi-agent setting fails whereas our Partner-Aware UCB achieves sublinear regret.

\begin{figure}[h]
    \centering
    \includegraphics[width=\columnwidth]{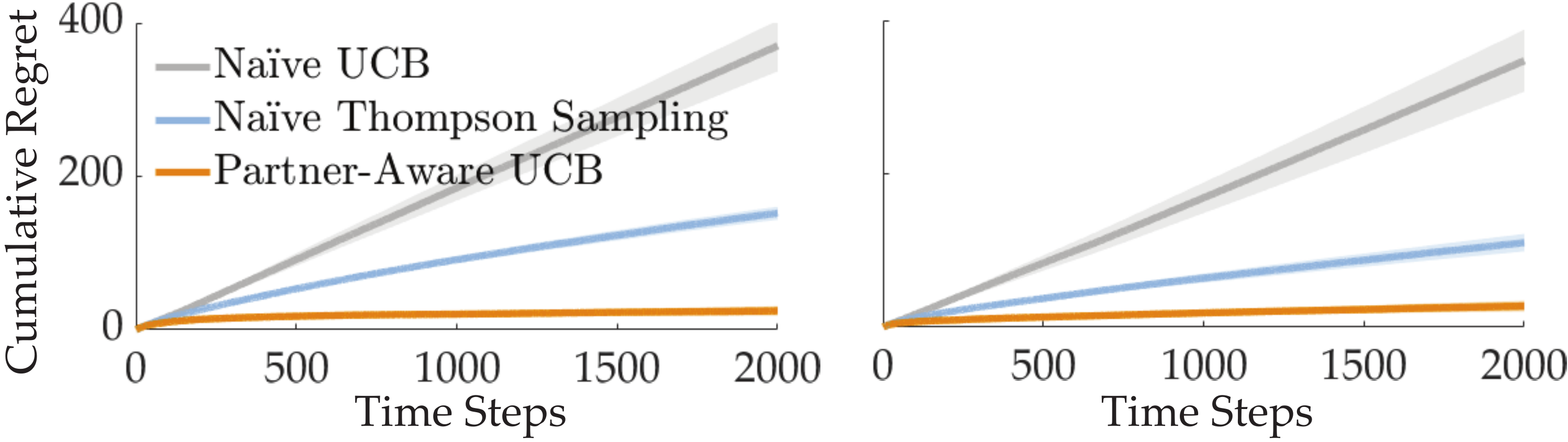}
    \caption{Cumulative regrets over $100$ runs for different algorithms with (left) fixed and (right) random reward means.}
    \label{fig:sim1}
\end{figure}

For this, we ran two simulations: one with fixed reward means (the same as before), and one where the reward means are generated randomly from a $\textrm{Unif}[0,1]$ prior. For Partner-Aware UCB, we set $L=1$, $W=25$. 
Figure~\ref{fig:sim1} shows the results. While Na\"ive UCB and Na\"ive Thompson Sampling result in linear regrets, our algorithm achieves sublinear regret in both cases. This result provides strong empirical evidence for our claim and demonstrates the importance of partner-awareness and partner-modeling. Additional simulations are presented in Appendix~\ref{appendix:very_naive_ucb}.

\smallskip
\noindent\textbf{Varying Other Conditions.}\label{subsec:sim2}
Having demonstrated the success of Partner-Aware UCB, we investigate its performance under varying conditions. Specifically, we check the effects of observability.

\begin{figure}[h]
    \centering
    \includegraphics[width=\columnwidth]{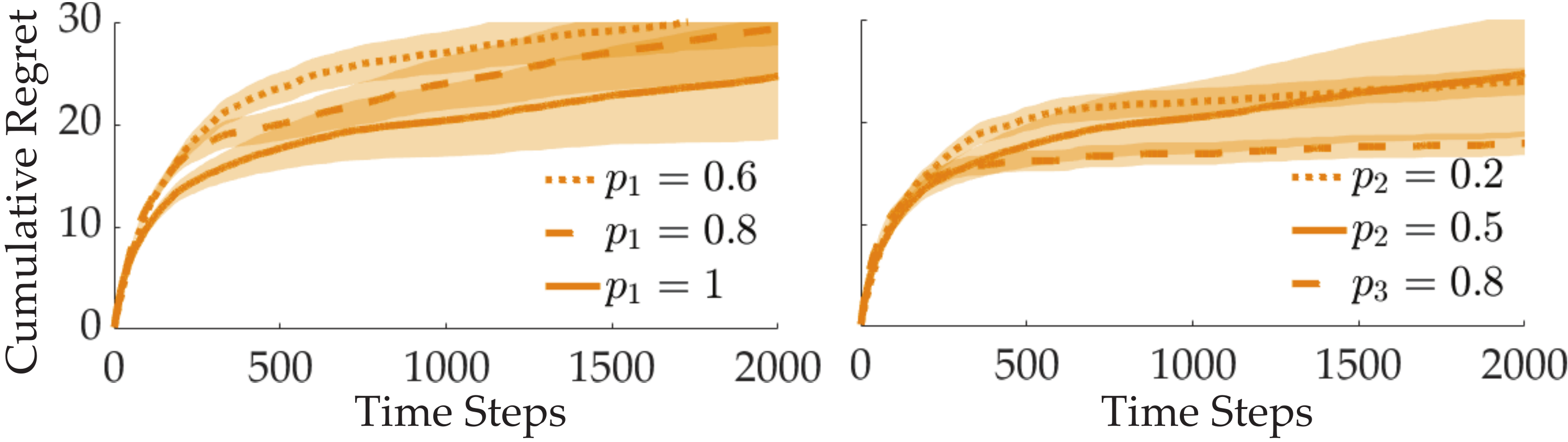}
    \caption{Cumulative regret values over $100$ runs under different (left) leader and (right) follower observabilities.}
    \label{fig:sim2}
\end{figure}

For this, we ran simulations with fixed reward means (the same as before) and vary $p_1\in\{0.6, 0.8, 1.0\}$, $p_2\in\{0.2,0.5,0.8\}$. We set $L=1$, and $W=25$.

Figure~\ref{fig:sim2} shows the results for both varying $p_1$ (left), and $p_2$ (right) experiments. In all cases, Partner-Aware UCB incurs only sublinear regret.\footnote{As the cumulative regret takes partial observability into account, it does not necessarily increase with lower observability.}

\begin{figure}[h]
    \centering
    \includegraphics[width=\columnwidth]{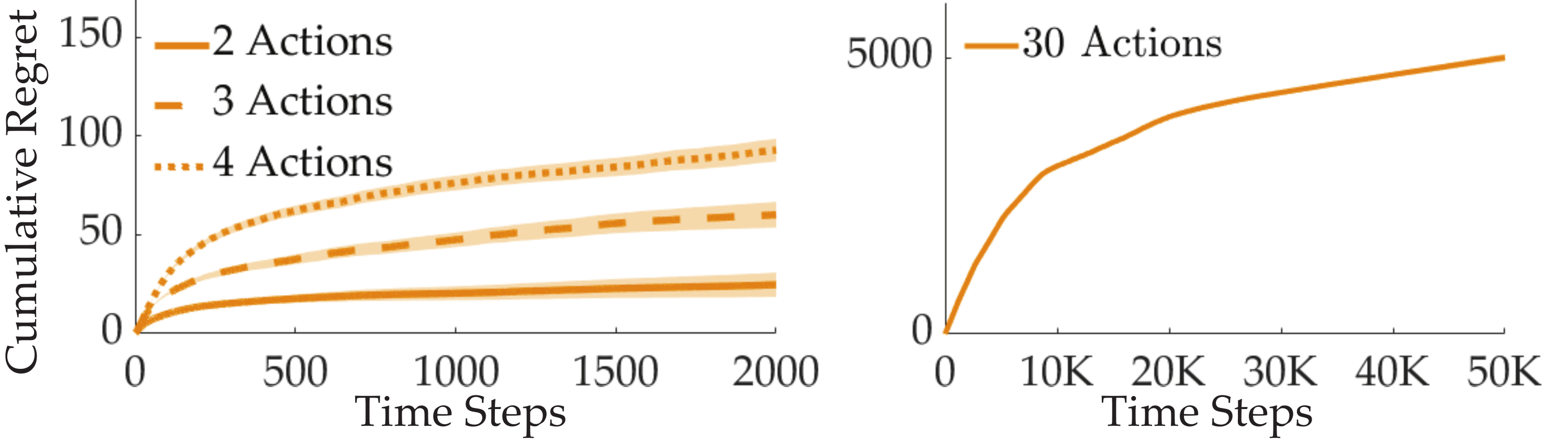}
    \caption{Average regret values over $100$ runs with varying number of available actions for both agents.}
    \label{fig:sim3}
\end{figure}

We also experiment with varying number of available actions to the agents. Figure~\ref{fig:sim3}~(left) shows the results for $|\mc{A}_1|=|\mc{A}_2|\in\{2,3,4\}$, averaged over $100$ runs.\footnote{Similar to the fixed reward values as in the two-action case, we designed the rewards such that there are $|\mc{A}_1|=|\mc{A}_2|$ local optima.} While the incurred regret naturally increases with higher number of actions, Partner-Aware UCB achieves sublinear regret in all cases. Due to different scale, we present the results with $30$ actions on a separate plot in Fig.~\ref{fig:sim3}~(right) under the random reward setting.

\begin{figure}[h]
    \centering
    \includegraphics[width=0.6\columnwidth]{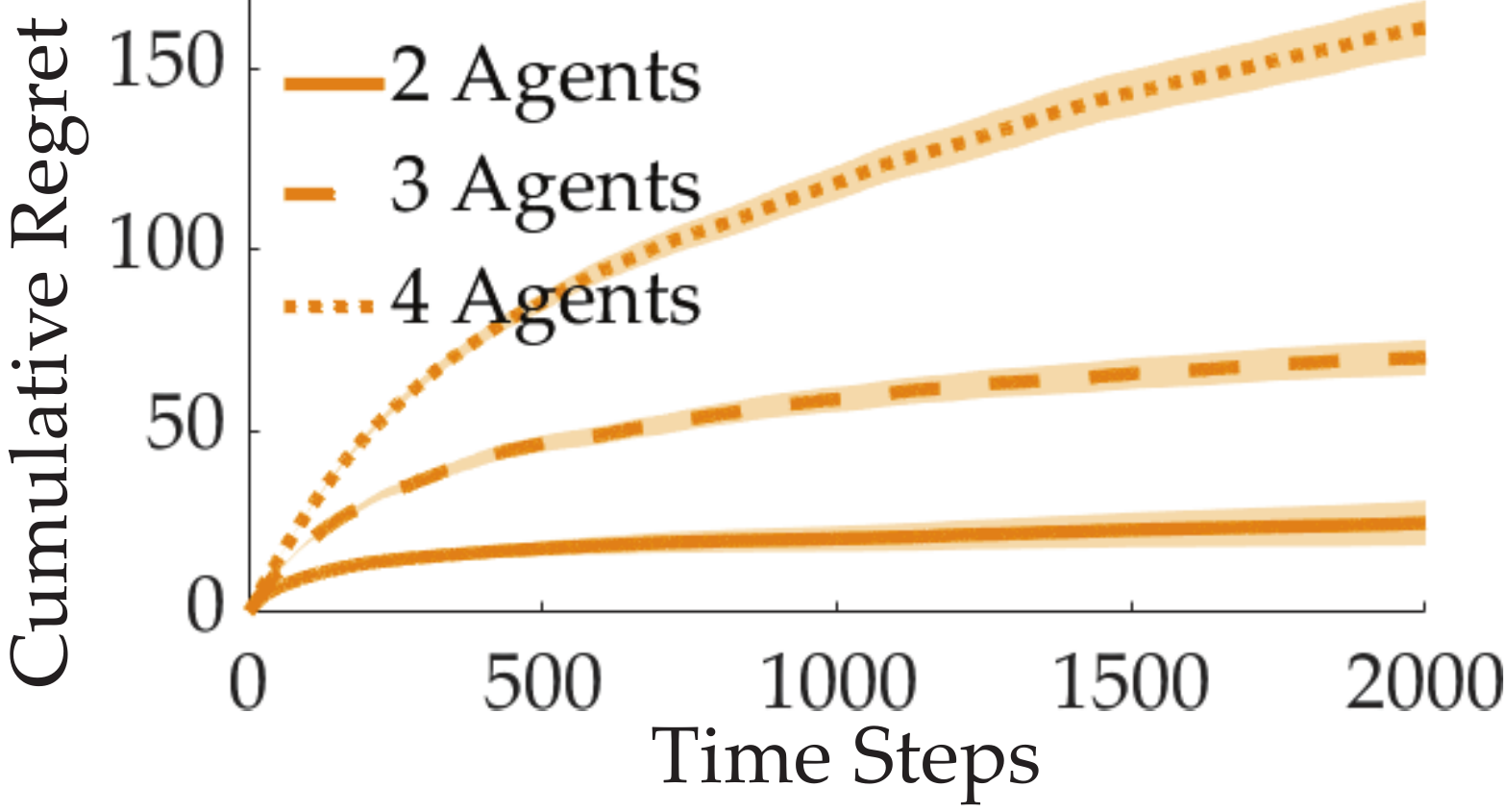}
    \caption{Average regret values over $100$ runs with varying number of agents with the extended algorithm.}
    \label{fig:sim4}
\end{figure}

\noindent\textbf{Generalization to More Agents.} We now generalize our algorithm to more than two agents, a useful formalism for applications in human-robot teams. For this, we first note there is a leader and a follower in the original algorithm, and only one of them models the other. The primary motivation for this is to avoid deadlock situations where agents oscillate between actions to ``catch" the other agent's behavior.

While agents should not model each other, it would also not be enough if they modeled only the agent with the highest reward observability. Even if they could accurately predict that agent's actions, they would still need to solve a decentralized MAB among themselves.

This informs us about the following recursive approach. Suppose we have an $N$-agent problem. As in the original Partner-Aware UCB, the agent with the highest observability does not model the others and tries to optimize its own action as if the others will comply. All other agents model and attempt to predict this leader agent. They now have to deal with an $(N-1)$-agent problem. Hence, the agent with the second highest observability does not model the remaining $N-2$ agents who, on the other hand, model this ``second leader". This hierarchy we impose based on observability continues until the problem reduces to a single-agent problem for the last agent.

To test if the extended algorithm achieves sublinear regret, we ran simulations with varying number of agents from $N\in\{2,3,4\}$, and $|A_i|=2$ for all agents and $p_i=i/N$. The reward means were fixed\footnote{Similar to the other experiments, we designed the reward values such that there are $|A_1|=\dots=|A_N|=2$ local optima.}, and we set $c^{(i)}=0.025$, $L=1$, and $W=25$ for all agents. Fig.~\ref{fig:sim4} shows the results averaged over $100$ runs. The extended Partner-Aware UCB achieves sublinear regret in all cases. 

\smallskip
\noindent\textbf{Generalization to Other Bandits.} The reason why Partner-Aware UCB performs successfully is that it allows the follower agent to learn its best individual action conditioned on the leader's action, which allows the agents to discover the team-optimal action. We hypothesize this idea could solve a broader class of decentralized cooperative bandit problems.

To test this, we simulated two settings: (i) a flipped setting where the agents get a reward of $1$ instead of $0$ with probability of $p_i$ (and still get $r^*_t(a)$ with probability $1-p_i$), (ii) a Gaussian bandits setting where $r^*_t(a)$ comes from an action-dependent stationary Gaussian distribution with unknown mean and variance, and agents get reward $r^*_t(a) + \nu_i$ where $\nu_i\sim\mathcal{N}(0,\sigma_i)$ for different $\sigma_i$ (we set the agent with higher $\sigma_i$ to be the follower as it has noisier rewards).

\begin{figure}[h]
    \centering
    \includegraphics[width=\columnwidth]{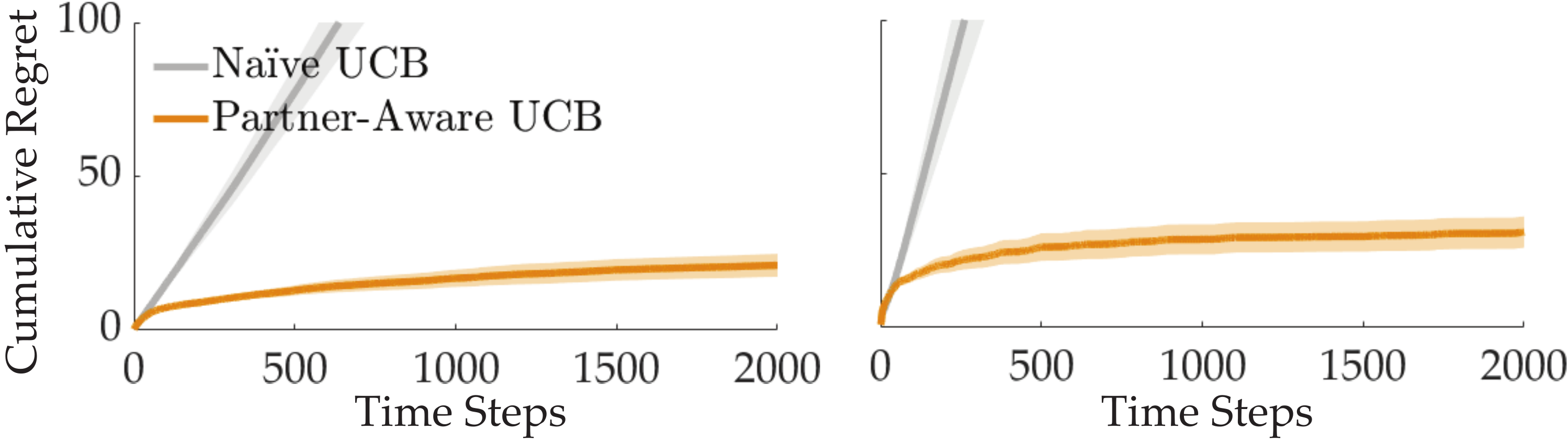}
    \caption{Average regret values over $100$ runs in the (left) flipped and (right) Gaussian settings.}
    \label{fig:other_envs}
\end{figure}

As it can be seen in Fig.~\ref{fig:other_envs}, where we ran simulations with random reward means (and random std for $r^*_t\in[0.1,0.5]$ in Gaussian bandits with $\sigma_1=0.1$ and $\sigma_2=0.5$), Partner-Aware UCB outperforms Na\"ive UCB and achieves a sublinear regret even in these modified settings.

\section{Experiments}
\label{sec:exp}
% Our simulations in Sec.~\ref{sec:sim} demonstrate the success of our Partner-Aware UCB algorithm in a variety of settings including when partnered with agents that employ the same Partner-Aware UCB algorithm or at times with partners following the Na\"ive UCB algorithm (if the partner has the highest observability among the agents and $L=1$). However, to demonstrate its effectiveness, we evaluate how it performs with humans on collaborative tasks. 

We now empirically analyze our algorithm through an in-lab human-subject study where the participants collaborate with a robot arm to stack burgers. While this experiment involves a short horizon, we also present an online human-subject study where the participants collaborate with a robot for long horizons to maximize their profit on a grid of slot machines in Appendix~\ref{appendix:casino_section}. Our user studies have been approved by the local research compliance office. Subjects were compensated with \$15/hour for their participation.

% Both of these studies will test the effectiveness of the Partner-Aware UCB algorithm and demonstrate improved team performance in these collaborative tasks. 
% , we want to explore how it performs when collaborating with humans, which is crucial for our algorithm to be effective and successful in human-robot interaction. To empirically analyze our algorithm, we performed two user studies:
% one online study and one in-lab study with a 7-DoF robotic arm (Fetch \cite{wise2016fetch}, Fetch Robotics). While the online study has a game setup where the participants must collaborate with the computer to earn as much reward as possible from a grid of slot machines, the in-lab study with the robot focuses on collaborative task where the participants work with the robot the prepare burgers. In both studies, we verify that the Partner-Aware UCB algorithm improves the team performance in collaborative tasks.

%\begin{figure}[ht]
%    \centering
%    \includegraphics[width=\columnwidth]{figs/interface.png}
%    \caption{User interfaces for (left) the online casino study and (right) in-lab burger stacking robot experiment. Green and red numbers indicate successful and unsuccessful trials. Yellow highlights show the most recent team action.}
%    \label{fig:interface}
%\end{figure}

\smallskip
\noindent\textbf{Experimental Setup.} We designed a collaborative burger stacking experiment as shown in Fig.~\ref{fig:frontfig}.\footnote{Video at: \url{https://sites.google.com/view/partner-aware-ucb}}
Subjects were told they work at a burger store with a robot to stack burgers. They are responsible for placing the patty and the cheese, whereas the robot is for the tomatoes and lettuce. They decide whether the patty or the cheese should go on top of the bottom bun, and the robot decides the second layer. Decisions are simultaneous without knowing each other's action.

The participants were told there is a fixed probability associated with whether a customer liked the burger. After stacking each burger, the robot and the human are informed about if a customer was satisfied. The robot has a sensor defect and observes only half of the satisfied customers. It senses the others as unsatisfied (human's observability is $p_1\!=\!1$ and robot's $p_2\!=\!0.5$). The goal of both the human and the robot is to maximize the number of satisfied customers.

\smallskip
\noindent\textbf{Independent Variables.} We varied the robot's algorithm: Na\"ive UCB and Partner-Aware UCB. We set, when relevant, $L=1$, $W=2$ and $c^{(L)}=c^{(F)}=0.01$.

\smallskip
\noindent\textbf{Procedure.} We conducted a within-subjects study with a Fetch robot \cite{wise2016fetch} for which we recruited $58$ participants ($22$ female, $36$ male, ages 18 -- 69).
Due to the pandemic conditions, the first five of the subjects participated the study with a real robot in the lab, and the rest participated remotely with an online interface.
The participants interacted with the robot to prepare $40$ burgers together, $20$ with each algorithm. The participants knew the number of burgers they are going to prepare in advance.

Initially, MAB requires significant exploration, so comparison between the two algorithms at early stages will not yield any meaningful results. However, evaluating later stages of collaboration would require many repeated long-term interactions with the robot, which is not feasible due to limitations on the duration of in-lab studies with a robot. Instead, we warm-start each algorithm by allowing them to collaborate with a simulated Na\"ive UCB agent for stacking $20$ burgers to proceed forward in the exploration stage so that the robot's algorithm will be more critical for performance. After these $20$ burgers, the simulated agent is replaced with the study participant for preparing $20$ more burgers with each algorithm. 
% Without any prior knowledge, exploration is initially the most crucial thing in MAB. Therefore, both Partner-Aware and Na\"ive UCB would perform random exploration actions in the initial time steps. As we were limited in the experiment duration due to the robot speed and the existing pandemic conditions, this could yield very similar outcomes for both algorithms. Hence, we first let a simulated Na\"ive UCB agent interact with Fetch robot (equipped with either Na\"ive UCB or Partner-Aware UCB) for preparing $20$ burgers with each algorithm. The purpose of this simulated agent was to proceed forward in the exploration stage so that the algorithm used will be more critical for performance. After these $20$ burgers, the simulated agent was replaced with the study participants for preparing $20$ more burgers with each algorithm. 

The user interface aided the participants by providing information about: the number of satisfied and unsatisfied customers for each burger configuration, the total number of burgers stacked, the configuration of the latest burger and whether it made the customer satisfied.

For a fair comparison, we randomized the reward means only between the users and not between the algorithms.
We swapped the actions to prevent participants from realizing they are dealing with the same problem instance. Hence, between the two sets, for example, $\mu_{(0,1)}$ of the first set was equal to $\mu_{(1,0)}$ in the second. 
To further avoid any bias due to ordering, half of the participants first worked with Na\"ive UCB and the other half with the Partner-Aware UCB.

\smallskip
\noindent\textbf{Dependent Measures.} 
% To quantitatively test the performance of each algorithm, 
We measured cumulative regret and the total number of satisfied customers. We excluded the first $20$ simulated time steps for a fair comparison. Additionally, the participants took a $5$-point rating scale survey ($1$-Strongly Disagree, $5$-Strongly Agree) consisting of $5$ questions for each algorithm: ``I was usually able to stack the burger I wanted" (\emph{Ability}), ``The robot insisted on some suboptimal burgers" (\emph{Insisting}), ``The robot was easy to collaborate with" (\emph{Easy}), ``The robot was annoying" (\emph{Annoying}), and ``I could get more happy customers if I were stacking burgers alone" (\emph{Alone}).

\smallskip
\noindent\textbf{Hypotheses.}
\begin{quote}
    \textbf{H1.} \textit{Users interacting with Partner-Aware UCB robot will incur smaller regret and keep the customers more satisfied.}\\
    \textbf{H2.} \textit{Users will subjectively perceive the Partner-Aware UCB robot as a better partner who can effectively collaborate with them.}
\end{quote}

\noindent\textbf{Results-Objective.}
%The Partner-Aware UCB robot conditions its own action on the predicted human action to successfully minimize the regret. On the other hand, Na\"ive UCB fails to do that by implicitly assuming the human will comply with its decision. Hence, 
Partner-Aware UCB achieves lower regret ($2.7\pm0.31$) compared to Na\"ive UCB ($3.6\pm0.31$) with statistical significance ($p<.005$). Fig.~\ref{fig:survey_robot}~(left) shows the cumulative regret incurred over time with both algorithms.

The significant difference in the cumulative regret is also reflected in the number of satisfied customers supporting \textbf{H1}: Partner-Aware UCB achieved significantly higher number of satisfied customers ($13.5\pm0.6$) than Na\"ive UCB ($12.7\pm0.6$), with $p<.05$.

\begin{figure}[h]
    \centering
    \includegraphics[width=\columnwidth]{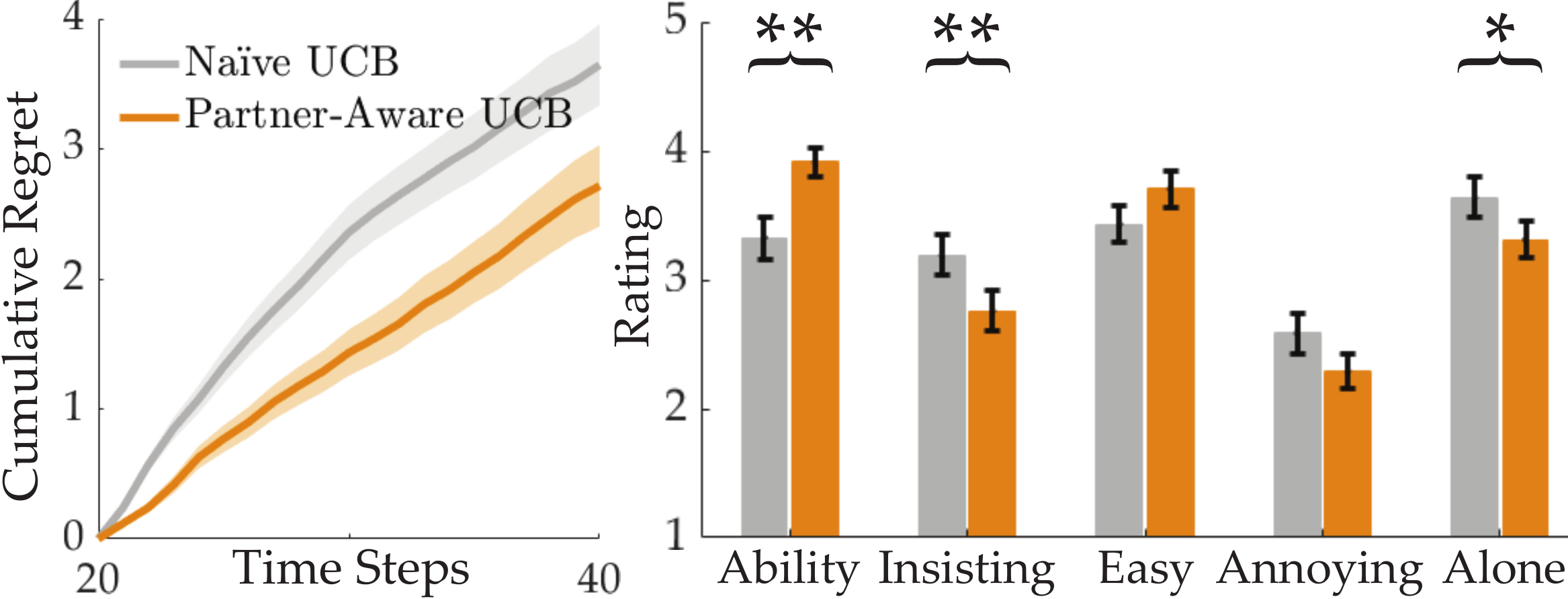}
    \caption{(Left) Average regret over time, (right) survey results for the burger-stacking robot experiment. Single and double asterisks indicate $p\!<\!.05$ and $p\!<\!.005$, respectively.}
    \label{fig:survey_robot}
\end{figure}

\noindent\textbf{Results-Subjective.} The Na\"ive UCB robot often decides stacking an under-explored burger and insists on the same action until that burger is made. While this occasionally helps the humans to exploit a burger, it often causes deadlock situations where both agents are unable to stack their intended burgers. On the other hand, Partner-Aware UCB keeps a model of the human, and avoids such situations. We believe this explains the subjective preferences of the users.

We plot the users' survey responses in Fig.~\ref{fig:survey_robot}~(right). The responses were reliable with Cronbach's alpha $ > 0.95$. The users indicated they were able to stack the burger they wanted (\emph{Ability}) more frequently with the Partner-Aware UCB ($p<.005$), and thought it was easier to collaborate with (\emph{Easy}, $p\approx.07$), whereas found the Na\"ive UCB robot more annoying (\emph{Annoying}, $p\approx.05$). They also indicated the Na\"ive UCB insisted more on the suboptimal burgers (\emph{Insisting}, $p<.005$). Finally, the users think they could have more satisfied customers if they were stacking burgers alone with a higher confidence when partnered with the Na\"ive UCB robot (\emph{Alone}, $p<.05$). These results strongly support \textbf{H2}. We further analyze and discuss how different populations of human users perform differently in Appendix~\ref{appendix:burger_stacking_additional}.

\section{Conclusion}
\label{sec:conclusion}
\noindent\textbf{Summary.} We studied multi-agent decentralized MAB, where the reward obtained by the team depends on all agents' actions. We showed na\"ive extensions of optimal single-agent MAB algorithms -- where each agent disregarded others' actions -- fail when rewards are coupled. We proposed a simple yet powerful algorithm for partners to model and coordinate with the partners who have higher observability over the task. Our algorithm only relies on the observation of partner's actions and accomplishes the coordination without explicit communication. We analytically showed it achieves logarithmic regret and tested our hypotheses through simulations and experiments.

\smallskip
\noindent\textbf{Limitations and Future Work}. The decentralized MAB is a useful abstraction for many real-world coordination tasks, and we are excited that our algorithm yet simple demonstrates significant improvements to enable seamless coordination. However, many applications require more complex formulations such as Markov Decision Processes. In the future, we plan to extend the intuitions gained by our algorithm and analysis to some of these more complex settings.

Another interesting direction is pairing the partner-aware strategy with algorithms other than UCB, e.g., Thompson sampling. Our preliminary results indicate it still gives significant improvements over the na\"ive counterparts.
%summary and discussion / limitations/ future work

%Limitation: Robot study captures only short-time interactions.

\section*{Acknowledgments}
The authors acknowledge funding from NSF Awards \#1941722, \#2006388, and \#2125511, Office of Naval Research, and Air Force Office of Scientific Research.

\bibliography{refs}

\appendix\onecolumn
\section*{Appendix}
In the Appendix, we first discuss the effect of parameters $W$ and $L$ in our Partner-Aware UCB algorithm (Appendix~\ref{appendix:effect_of_w_and_l}). We then present the proof of Theorem~\ref{thm:main_theorem} in Appendix~\ref{appendix:main_proof}. Appendix~\ref{appendix:ucb_matched_with_kg} presents the simulation results for what happens when a Partner-Aware UCB follower collaborates with a leader who employs knowledge gradient algorithm. Appendix~\ref{appendix:very_naive_ucb} gives additional simulation results where agents completely ignore the multi-agent aspects of the problem. In Appendix~\ref{appendix:burger_stacking_additional}, we make further analysis on the burger stacking robot experiments that show how different populations of human users perform differently. Appendix~\ref{appendix:casino_section} demonstrates the effectiveness of Partner-Aware UCB in long-term human-robot collaboration through an online human-subject study. Finally, Appendix~\ref{appendix:infrastructure} presents the computation infrastructure we used for our simulations and experiments.

\section{Effect of parameters $W$ and $L$}
\label{appendix:effect_of_w_and_l}
The parameter $W$ denotes the number of leader's past actions used by the follower to compute the sampling distribution $\tilde{\rho}^{(L)}_t$ at any time $t$. Larger values of $W$ imply that the follower accounts for more number of past actions by the leader and hence $\tilde{\rho}^{(L)}_t$, i.e., the follower's predicted leader action is less sensitive to its recent actions. When $W=1$, the follower only looks at the leader's latest action and predicts that the leader will repeat its action. As shown in Fig~\ref{fig:sim0}, regret incurred when $W=1$ is smaller than any $W> 1$. This is because the leader uses the upper confidence bound over each action, which is computed using the entire past history. We also note that $W=1$ may not be optimal when humans are involved, because unlike our partner-aware leader algorithm, humans tend to be more myopic in their decision making. The parameter $L$ denotes the number of times the leader repeats its local action. In all our experiments we fix $L=1$, and to simplify our analysis we assume $L>1$, specifically $L=2$. 

Before we provide the proof of our main result, we provide the following fact which shows that agents do not need to model each other when all agents have full reward observability.

\begin{fact}
Let $|\mc{A}_1| \!=\! |\mc{A}_2| \!=\! K$ and define $\Delta_a\!:=\! \mu_{a_{\ast}}\!-\!\mu_a$ for all $a \!\in\! \mc{A}$. Consider a decentralized team where $p_1 = p_2 = 1$, i.e., both agents have full reward observability. Then, the team reduces to a single agent MAB. Implementing UCB algorithm at single MAB agent achieves logarithmic regret~\cite{lattimore_szepesvari_2020}. Hence, if each agent implements the UCB algorithm locally, applying Theorem~7.1 in~\cite{lattimore_szepesvari_2020}, the team achieves logarithmic cumulative regret.
%which can be upper bounded as $\reg(T) \leq O\left(\frac{K^2}{\Delta} \log T \right)$, where $\Delta:= \min_{a\neq a_{\ast}} \Delta_a$.% denotes the gap of the second best arm from the optimal arm. 
\end{fact}

However, when there is an agent with partial reward observability, if they do not model each other then the regret grows linearly with time, as we have seen in Fig.~\ref{fig:sim1}.

\section{Proof of Main Theorem}
\label{appendix:main_proof}
Let $u_{(i,j)}$ denote a positive integer to be defined later for each team action $(i,j) \in \mc{A}$. For all rows $i$ in leader's action space $\mc{A}_L$, define the optimal column with highest pay-off as
\begin{align}
    j^{\ast}(i) := \argmax_{j \in \mc{A}_F} \mu_{(i,j)}.
\end{align}
Define the following good event
\begin{align}
    G^{(F)}_{i}
    := & 
    \left\{
    \mu^{(F)}_{(i,j^{\ast}(i))} \!<\! \min_{t \in [T]} f^{(F)}_{(i,j^{\ast}(i))}(t, \delta)
    \right\}\bigcap \\
    & \left\{
    \bigcap_{j \in \mc{A}_F\setminus \{j^{\ast}(i)\}}
    \!\left\{
    \hat{\mu}^{(F)}_{(i, j)}(u_{(i, j)})+ \sqrt{\frac{2\log1/\delta}{u_{(i, j)}}}
    < \mu^{(F)}_{(i,j^{\ast}(i))}
    \right\}\right\}.
\end{align}
On the good event $G^{(F)}_{i}$, the mean value of optimal column in row $i$, i.e., $\mu^{(F)}_{(i, j^{\ast}(i))}$ will never be underestimated by the follower's upper confidence bound for the mean of action $(i, j^{\ast}(i))$. Furthermore, on event $G^{(F)}_{i}$ the follower's upper confidence bound obtained for the mean of action $(i,j)$ after $u_{(i, j)}$ observations are taken by the team is below the pay-off of the best action in the row $(i, j^{\ast}(i))$ when $j$ is a sub-optimal column. 

Recall that for the special case of $L=2$ in our partner-aware learning algorithm, leader takes each action twice. Thus, at odd time instants leader takes a new action according to its UCB and at even time instants it repeats the same action. Since $W = 1$, the follower predicts the leader's action correctly at every even time instant.

\begin{lemma}
\label{lemma:row_ucb}
Conditioned on the event $\bigcap_{i \in \mc{A}_L} G^{(F)}_i$, on even time instants the row sub-optimal columns will not be chosen by the team for more than $\sum_{i \in \mc{A}_L}\sum_{j \neq j^{\ast}(i)}u_{(i,j)}$ times.
\end{lemma}
\begin{proof}
Suppose event $G^{(F)}_{i}$ holds true. Suppose there exists some even time instant $t$ where leader chooses row $i$, we have $n_{(i, j)}(t-1) = u_{(i, j)}$ for all sub-optimal columns $j \in \mc{A}_F\setminus \{j^{\ast}(i)\}$ and follower chooses a sub-optimal column $j \in \mc{A}_F \setminus \{j^{\ast}(i)\}$. Hence, team action $a_{t} = (i, j)$ gets played at time $t$. Then, we get
\begin{align}
    f^{(F)}_{(i, j)}(t, \delta)
    & = \hat{\mu}^{(F)}_{(i, j)}(t-1) + \sqrt{\frac{2\log 1/\delta}{n_{(i, j)}(t-1)}} \\
    & \overset{(a)}= \hat{\mu}^{(F)}_{(i, j)}(u_{(i, j)}) + \sqrt{ \frac{2\log 1/\delta}{u_{(i, j)}}} \overset{(b)} < \mu^{(F)}_{(i,j^{\ast}(i))} \!\overset{(c)} <\! f^{(F)}_{(i,j^{\ast}(i))}(t, \delta),
    \label{eq:row_ucb_contradiction}
\end{align}
where $(a)$ follows from the assumption that $n_{(i, j)}(t-1) = u_{(i, j)}$, $(b)$ and $(c)$ follow from the definition of the event $G^{(F)}_{i}$. The inequality in~\eqref{eq:row_ucb_contradiction} is a contradiction to the fact that a sub-optimal column $j$ of row $i$ was played at time $t$. In other words, sub-optimal action $(i,j)$ on the event $G^{(F)}_i$ will be played at most $u_{(i,j)}$ times given the leader chooses the sub-optimal row $i$. Hence, the sub-optimal columns of the action matrix will not be played for more than $\sum_{i \in \mc{A}_L}\sum_{j \neq j^{\ast}(i)} u_{(i,j)}$ on even time instants. 
\end{proof}

After $\sum_{j \neq j^{\ast}(i)} u_{(i,j)}$ even time instants, for all future even time instants whenever the leader chooses row $i$, follower will choose the optimal column $j^{\ast}(i)$. Now, we want to show that the leader explores the action $(i, j^{\ast}(i))$ for at most $u_{(i, j^{\ast}(i))} $ time steps. For all $i \in \mc{A}_L \setminus \{a^{(L)}_{\ast}\}$, define 
\begin{align}
    G^{(L)}_{i}
    = 
    &\left\{
    \mu^{(L)}_{(i,j^{\ast}(i))} < \min_{t \in [T]} f^{(L)}_{(i,j^{\ast}(i))}(t, \delta)
    \right\} \bigcap \\
    & \left\{
    \bigcap_{j \in \mc{A}_F \setminus \{j^{\ast}(i)\}}
    \left\{
    \hat{\mu}^{(L)}_{(i, j)}(u_{(i, j)})+ \sqrt{\frac{2\log1/\delta}{u_{(i, j)}}}
    < \mu^{(L)}_{(i,j^{\ast}(i))}
    \right\}\right\}\bigcap 
    \\
    & \left\{
    \mu^{(L)}_{a^{\ast}} < \min_{t \in [T]} f^{(L)}_{a^{\ast}}(t, \delta)
    \right\} \bigcap
    \left\{
    \hat{\mu}^{(L)}_{(i, j^{\ast}(i))}(u_{(i, j^{\ast}(i))})+ \sqrt{\frac{2\log1/\delta}{u_{(i, j^{\ast}(i))}}}
    < \mu^{(L)}_{a^{\ast}}
    \right\}.
\end{align}
The good event $G^{(L)}_{i}$ for the leader has events similar to the good event $G^{(F)}_{i}$ for follower and some additional events. On the event $G^{(L)}_{i}$, the mean value of optimal column in row $i$, i.e., $\mu^{(L)}_{(i, j^{\ast}(i))}$ will never be underestimated by the leader's upper confidence bound for the mean of action $(i, j^{\ast}(i))$. Furthermore, on event $G^{(L)}_{i}$ the leader's upper confidence bound obtained for the mean of action $(i,j)$ after $u_{(i, j)}$ observations are taken by the team is below the pay-off of the best action in the row $(i, j^{\ast}(i))$ when $j$ is a sub-optimal column. Additionally, on this event the leader's upper confidence bound for the optimal action is never underestimated by the leader and the leader's upper confidence bound for the optimal action in the row is below the pay-off of the optimal team action.

\begin{lemma}
\label{lemma:leader_ucb}
Conditioned on the event $\bigcap_{i \neq a^{(L)}_{\ast}} G^{(L)}_i \bigcap G^{(F)}_i$, the sub-optimal rows will be chosen for at most $2\sum_{i \neq a^{(L)}_{\ast}} \sum_{j \in \mc{A}_F} u_{(i,j)}$ time instants.
\end{lemma}
\begin{proof}
From Lemma~\ref{lemma:row_ucb}, we know that conditioned on the event $G^{(L)}_i \bigcap G^{(F)}_i$ there exists a time instant $t$ such that $n_{(i,j)}(t-1) = u_{(i,j)}$ for all $j \in \mc{A}_F\setminus \{j^{\ast}(i)\}$ and $n_{(i, j^{\ast}(i))}(t-1) = u_{(i, j^{\ast}(i))}$. This is true because even if the optimal column in row $i$ is never explored until all the sub-optimal columns are explored, we know that the sub-optimal columns will be explored for at most $\sum_{j \neq j^{\ast}(i)} u_{(i,j)}$ even time instants. After this point, the follower will choose the optimal column at even time instants whenever the leader chooses row $i$. Hence, now we want to show that optimal column of a sub-optimal row $i$ will be explored at most $u_{(i, j^{\ast}(i))}$ times. Now suppose at time $t$, team chooses action $(i,j)$, this will happen if the leader chooses an action $(i,j')$ for some $j' \in \mc{A}_F\setminus \{j^{\ast}(i)\}$ or if the leader chooses the optimal column in the row $(i, j^{\ast}(i))$. Consider the first case, then we get
\begin{align}
    f^{(L)}_{(i, j')}(t, \delta)
    & = \hat{\mu}^{(L)}_{(i, j')}(t-1) + \sqrt{\frac{2\log 1/\delta}{n_{(i, j')}(t-1)}}
    \\
    & \overset{(a)}= \hat{\mu}^{(L)}_{(i, j')}(u_{(i, j')}) + \sqrt{ \frac{2\log 1/\delta}{u_{(i, j')}}} \overset{(b)} < \mu^{(L)}_{(i,j^{\ast}(i))} \overset{(c)} < f^{(L)}_{(i,j^{\ast}(i))}(t,\delta),
    \label{eq:leader_ucb1}
\end{align}
where $(a)$ follows from the assumption that $n_{(i, j')}(t-1) = u_{(i, j')}$, $(b)$ and $(c)$ follow from the definition of the event $G^{(L)}_{i}$. The inequality in~\eqref{eq:leader_ucb1} is contradiction to the fact that the leader chose a sub-optimal column in the row. Now consider the second case where the leader chooses the optimal column $(i, j^{\ast}(i))$. Then, we get
\begin{align}
    f^{(L)}_{(i, j^{\ast}(i))}(t, \delta)
    & = \hat{\mu}^{(L)}_{(i, j^{\ast}(i))}(t-1) + \sqrt{\frac{2\log 1/\delta}{n_{(i,j^{\ast}(i))}(t-1)}}
    \\
    & \overset{(a)}= \hat{\mu}^{(L)}_{(i, j^{\ast}(i))}(u_{(i, j^{\ast}(i))}) + \sqrt{ \frac{2\log 1/\delta}{u_{(i, j^{\ast}(i))}}} \overset{(b)} < \mu^{(L)}_{a^{\ast}} \overset{(c)} < f^{(L)}_{a^{\ast}}(t,\delta),
    \label{eq:leader_ucb2}
\end{align}
where $(a)$ follows from the assumption that $n_{(i, j^{\ast}(i))}(t-1) = u_{(i, j^{\ast}(i))}$, $(b)$ and $(c)$ follow from the definition of the event $G^{(L)}_{i}$. The inequality in~\eqref{eq:leader_ucb2} is contradiction to the fact that the leader chose the optimal column in the row.  In other words, it takes $\sum_{i \neq a^{(L)}_{\ast}} \sum_{j \in \mc{A}_F} u_{(i,j)}$ time instants to never choose sub-optimal rows in future time instants. In the worst case, it will take at most $2\sum_{i \neq a^{(L)}_{\ast}} \sum_{j \in \mc{A}_F} u_{(i,j)}$ time instants for the team to explore all the sub-optimal rows. 
\end{proof}

Define an overall good event for the team as 
\begin{align}
    G:= \left\{\bigcap_{i \neq a^{(L)}_{\ast}} G^{(L)}_i \bigcap G^{(F)}_i \right\} \bigcap G^{(F)}_{a^{(L)}_{\ast}}.
\end{align}
Then, we can write
\begin{align}
    \expe\left[\sum_{a \neq a_{\ast}}n_{a}(T)\right] &= \expe\left[\sum_{a \neq a^{\ast}} n_{a}(T)\indicate{G}\right] + \expe\left[\sum_{a \neq a^{\ast}} n_{a}(T)\indicate{G^c}\right]
    \\
    & \overset{(a)}\leq 2 \sum_{i \neq a^{(L)}_{\ast}}\sum_{j \in \mc{A}_F} u_{(i,j)}
    + 2 \sum_{j \neq a^{(F)}_{\ast}} u_{(a^{(L)}_{\ast},j)}
    + T \P(G^c),
    \label{eq:num_sub_opt_arms}
\end{align}
where $(a)$ follows from Lemma~\ref{lemma:row_ucb} and Lemma~\ref{lemma:leader_ucb}.

Next, we choose the constants $u_{(i,j)}$ for all team actions $(i,j)$. For any $i \in \mc{A}_L$ and $j \neq j^{\ast}(i)$, set 
\begin{align}
    u_{(i,j)} = \left\lceil\frac{8\log 1/\delta}{p^2_{\text{min}}\tilde{\Delta}^2_{(i,j)}} \right\rceil,
\end{align}
where we define $\tilde{\Delta}_{(i,j)} = \mu_{(i,j^{\ast}(i))}- \mu_{(i,j)}$ and for $i \neq a^{(L)}_{\ast}$, set
\begin{align}
    u_{(i,j^{\ast}(i))} = \left\lceil\frac{8\log 1/\delta}{p^2_{\text{max}}\Delta^2_{(i,j^{\ast}(i))}} \right\rceil,
\end{align}
where recall that $\Delta_{(i,j^{\ast}(i))} = \mu_{a^{\ast}}- \mu_{(i,j^{\ast}(i))}$. Similar to the analysis of UCB algorithm presented in Theorem~7.1~in~\cite{lattimore_szepesvari_2020}, using Chernoff bound along with union bound we get
\begin{align}
    \P(G^c) 
    \leq & \sum_{i \neq a^{(L)}_{\ast}} \left\{T\delta + \exp\left( -\frac{u_{(i,j^{\ast}(i))}p^2_{\text{max}} \Delta^2_{(i,j^{\ast}(i))} }{8}\right)\right\} + \\
    &\sum_{i \in \mc{A}_L}\sum_{j \neq j^{\ast}(i)}\! \left\{T\delta + \exp\left(\! -\frac{u_{(i,j)}p^2_{\text{min}}\tilde{\Delta}^2_{(i,j)} }{8}\!\right)\right\}.
\end{align}
Substituting the above inequality in \eqref{eq:num_sub_opt_arms}, using $\delta = \frac{1}{T^2}$, we can write the cumulative regret for the team as follows 
\begin{align}
    &\reg(T) 
    = \sum_{a \neq a_{\ast}}\frac{(p_{\text{max}}+p_{\text{min}})}{2}\Delta_a \expe[n_{a}(T)] \leq \frac{(p_{\text{max}}+p_{\text{min}})\Delta_{\text{max}}}{2} \expe\left[\sum_{a \neq a_{\ast}}n_{a}(T)\right] 
    \\
    &\leq 
    (p_{\text{max}}+p_{\text{min}})\Delta_{\text{max}}\left[\sum_{i \neq a^{(L)}_{\ast}} \frac{16}{p^2_{\text{max}}\Delta^2_{(i, j^{\ast}(i))}} \log T + \sum_{i \in \mc{A}_L} \sum_{j \neq j^{\ast}(i)}\frac{16}{p^2_{\text{max}}\tilde{\Delta}^2_{(i, j)}}  \log T + \frac{3\lvert\mc{A}_L\rvert\lvert\mc{A}_F\rvert}{2}\right].
\end{align}

%\section{Proof of Additional Lemmata}
%\input{sec/additional_lemma}

\section{Partner-Aware UCB with Knowledge Gradient}\label{appendix:ucb_matched_with_kg}
The work by \cite{zhang2013cheap} has found that humans' behavior in multi-armed bandit problems are best captured by knowledge gradient algorithms among the many they have experimented. Therefore, we tested our Partner-Aware UCB follower agent with a leader agent who follows knowledge gradient. The idea behind knowledge gradient is as follows: the agent assumes it only has one free action left and it will then have to keep pulling the same arm which it thinks to be the optimal, i.e. it will keep exploiting. By calculating the expected return under the initial free action based on the posterior distributions at that time, the agent decides on what action to take. It then repeats the same procedure in all time steps.

\begin{figure}[h]
    \centering
    \includegraphics[width=.3\textwidth]{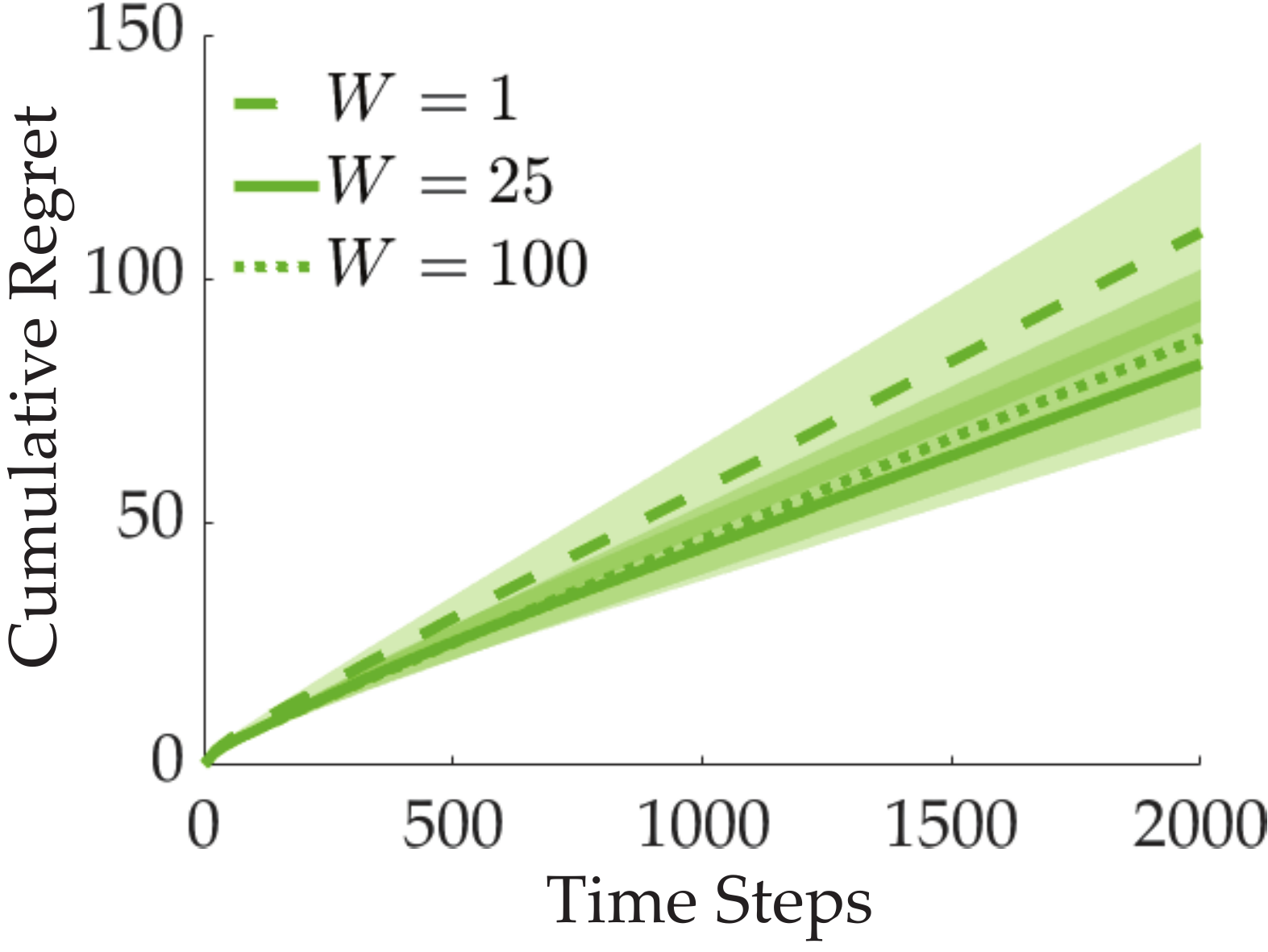}
    \caption{Regret values are averaged over $100$ runs for varying $W$ when Partner-Aware UCB follower is collaborating with a knowledge gradient leader.}
    \label{fig:apx1}
\end{figure}

Therefore, we ran simulations using the same setup as in Sec.~\ref{subsec:sim0}, and varying $W$. Fig.~\ref{fig:apx1} shows the incurred regret over time. Due to the suboptimalities of knowledge gradient, all three $W$ values led to linear regret. Nevertheless, $W=25$ performed the best. Moreover, our experiments presented in Sec.~\ref{sec:exp} empirically demonstrate humans often achieve sublinear regret. Further research may investigate the conditions when humans manage to find the optimal action in finite time.

\section{Additional Simulations}\label{appendix:very_naive_ucb}
One may wonder what happens if agents employ single-agent UCB that completely ignores the multi-agent aspects of the problem. Put another way, what happens if agents choose their actions as if the action space only consists of their actions, as opposed to Na\"ive UCB where agents are aware of the multi-agent formulation of the problem but assumes the other agent is going to comply? We name this version ``Very Na\"ive UCB", as it completely ignores the existence of the other agent. We implemented this as an additional baseline and ran simulations in the same setup as Fig.~\ref{fig:sim1} (right), but for a longer interaction to better highlight the difference. We present the results in Fig.~\ref{fig:vs_verynaiveucb}, which shows Partner-Aware UCB significantly outperforms Very Na\"ive UCB, too.

\begin{figure}[h]
    \centering
    \includegraphics[width=.3\textwidth]{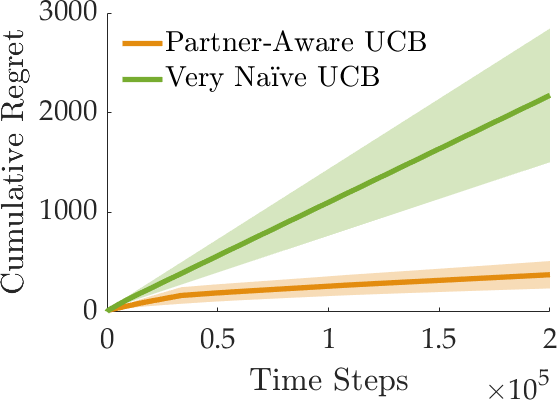}
    \caption{Cumulative regrets over $100$ runs for different algorithms with random reward means.}
    \label{fig:vs_verynaiveucb}
\end{figure}

\section{Additional Analysis on Burger Stacking Robot Experiments}
\label{appendix:burger_stacking_additional}
Having conducted the in-lab experiments with the actual robot first, we realized there is a significant difference between the team performances depending on whether the subject is an AI researcher or not. Specifically, 19 of the subjects are researchers in AI while the other 39 come from various other backgrounds. In Fig.~\ref{fig:researchers_vs_nonresearchers}, we show this difference.

\begin{figure}[h]
    \centering
    \includegraphics[width=.45\textwidth]{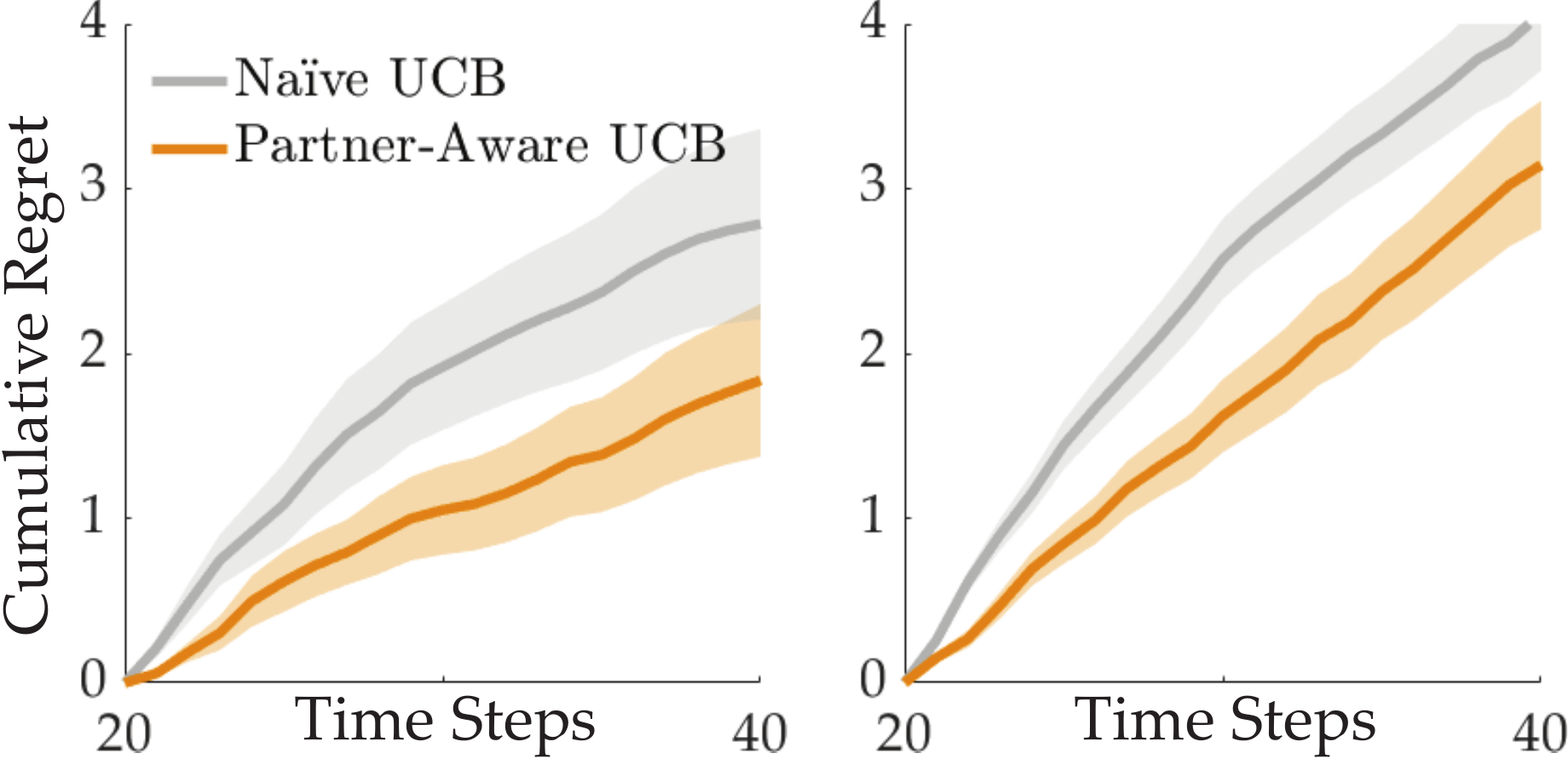}
    \caption{Cumulative regret over time for (left) AI researchers and (right) the others.}
    \label{fig:researchers_vs_nonresearchers}
\end{figure}

It can be seen that when paired with Na\"ive-UCB, AI researchers incur a cumulative regret of $2.8\pm0.58$ whereas other people incur $4.1\pm0.35$. Similarly, when paired with Partner-Aware UCB, the cumulative regret values are $1.8\pm0.47$ and $3.1\pm0.39$, respectively. In both cases, the teams with AI researchers perform significantly better ($p<0.05$, two-sample $t$-test).

This difference between the populations may imply that humans might be employing different algorithms or strategies depending on how familiar they are with the problem setup. While Partner-Aware UCB outperforms Na\"ive UCB in both cases, this observation opens new possibilities for research: it might be possible to develop and use different algorithms based on the end-users even in the environments that are as abstract as bandit problems.

\section{Online Casino Study}
\label{appendix:casino_section}
We conduct an online human-subject study to investigate longer interactions between the agents. In this experiment, the humans collaborate with the AI agent for $1000$ time steps at each episode.

\smallskip
\noindent\textbf{Experimental Setup.} We designed a simple casino interface with four slot machines, placed on a $2$-by-$2$ grid. %(see Fig.~\ref{fig:interface}~(left))
Human subjects were told they are in a casino with a budget of $1000$ units that should all be spent on these slot machines, and they can only select the row in the $2$-by-$2$ grid. The column is automatically selected by the AI agent which is not aware of the human's selection at that time, until the team's selection is revealed to both the human and the AI agent. Each human action costs $1$ unit and there is a fixed probability of earning a coin from each machine. After selecting the slot machine, the AI agent and the participant are informed about whether they earned a coin. However, the AI agent observes only $40\%$ of the coins and thinks the others resulted in no earnings (human's observability is $p_1=1$ and AI agent's $p_2=0.4$). The goal of both the human and the AI agent is to maximize the total number of coins earned together.

\smallskip
\noindent\textbf{Independent Variables.} We varied the algorithm the AI agent is using to collaborate with the human partner with two algorithms: Na\"ive UCB and Partner-Aware UCB. For both algorithms we set, when relevant, $L=1$, $W=5$ and $c^{(L)}=c^{(F)}=0.01$.

\smallskip
\noindent\textbf{Procedure.} We conducted an online within-subjects study with $24$ participants ($11$ female, $13$ male, ages 18 -- 58). None of the participants had prior experience with the experiment interface. Hence, they were given a chance to experience the setup in a trial casino whose reward means for each machine were randomly chosen. After the trial casino, each participant played in $50$ casinos ($25$ with each algorithm) in each of which they collaborated with the AI agent to select slot machines $1000$ times. The participants knew these numbers in advance, which potentially helped them in deciding on their exploration strategy. The keyboard controls helped them complete each casino within a minute.

  \begin{figure}[ht]
    \centering
    \begin{subfigure}[t]{.48\textwidth}
        \centering
        \includegraphics[width=\linewidth]{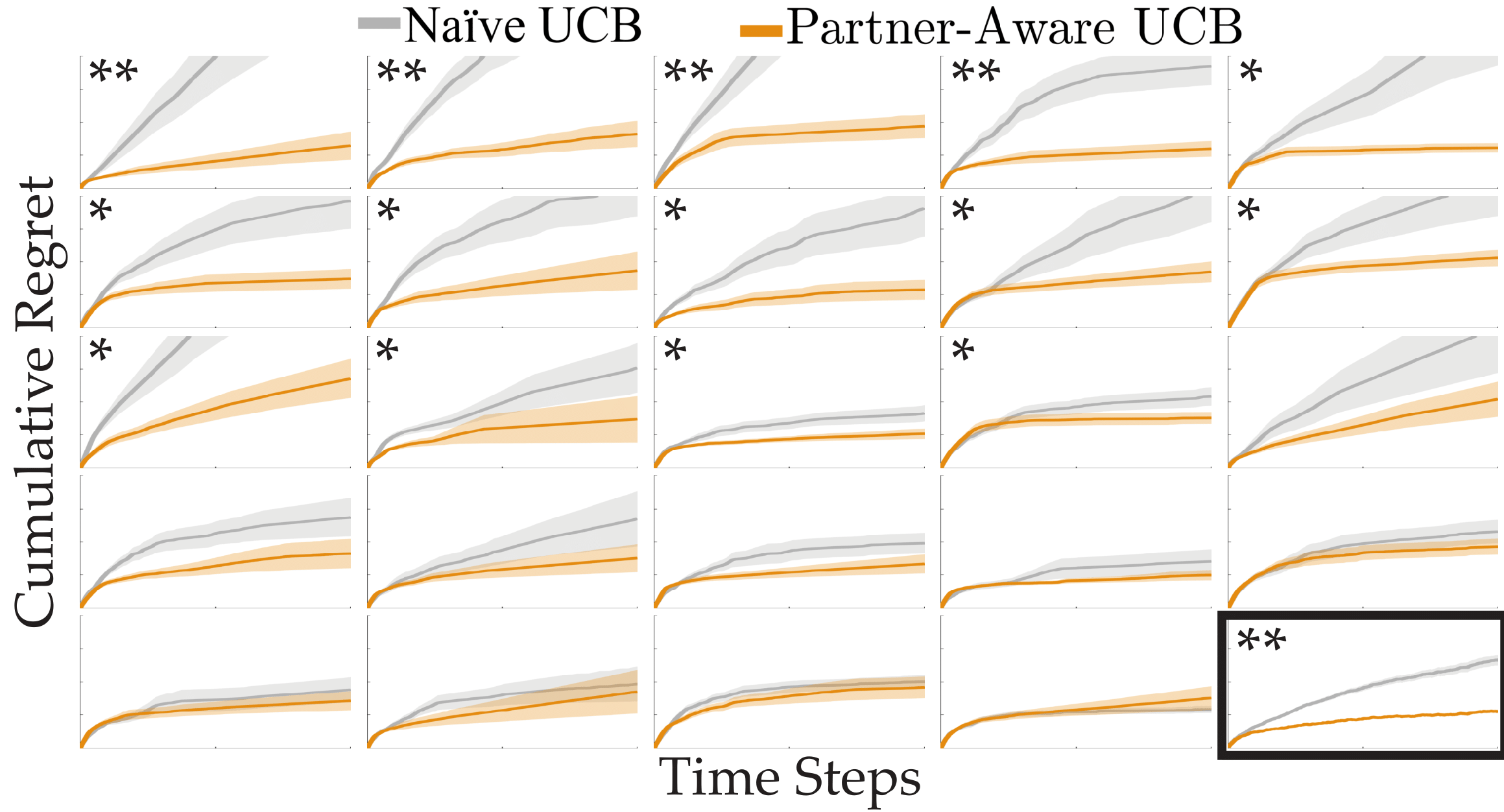}
        \caption{Cumulative regret over $25$ casinos: each plot shows the results of one user. The last plot is the average over both $25$ casinos and $24$ users.}
        \label{fig:users}
    \end{subfigure}%
    %%            %   Don't leave blank line
    \hfill        %%  <--- here
    \begin{subfigure}[t]{.48\textwidth}
        \centering
        \includegraphics[width=\linewidth]{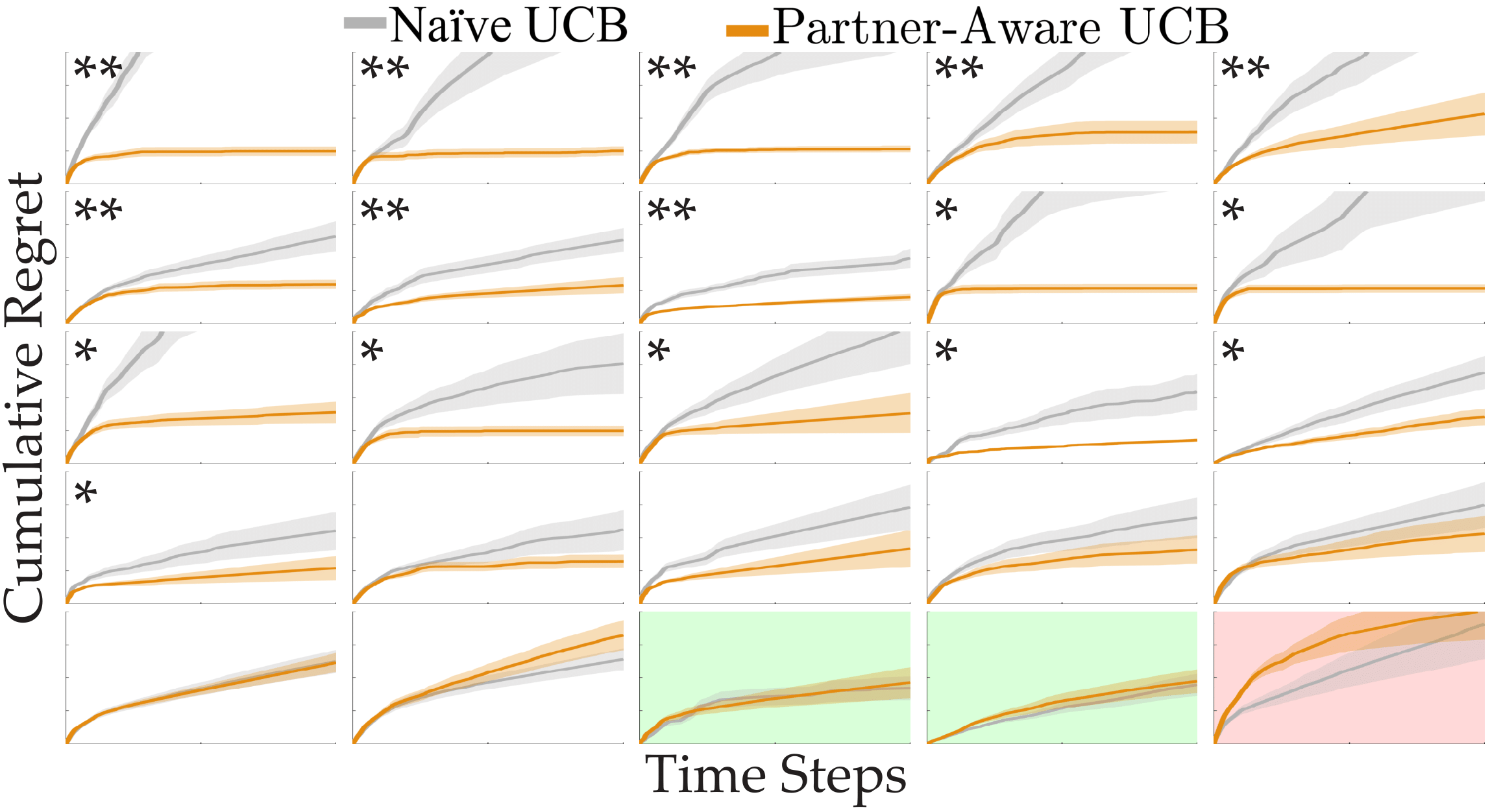}
        \caption{Cumulative regret over $24$ users: each plot shows the results of one casino with different reward means.}
        \label{fig:casinos}
    \end{subfigure}
\caption{Results of Online Casino Study}
\label{fig:casino_study_results}
\end{figure}

For fair comparison, we selected the same $r^*_t$ throughout the $1000$ time steps for the two sets of $25$ casinos. However, the participants did not know this, and the order of algorithms the participants partner with first was randomized.
% (resulting in $10$ participants playing with Na\"ive UCB first, and $14$ with Partner-Aware UCB first).

% To make the comparison between the algorithms fair, the casinos were exactly the same between the two sets of $25$ casinos, i.e. they had the same $r^*_t$ throughout the $1000$ time steps. However, the participants did not know this, and it is very unlikely that they discovered this given that they complete $24$ different casinos in between any two equivalent casinos. Nevertheless, the order of algorithms the participants partner with first was randomized (ending up $10$ participants playing with Na\"ive UCB first, and $14$ with Partner-Aware UCB).

The interface provided the participants the information about: the number of lucky (resulting in a coin) and unlucky selections for each machine, the total number of machines selected so far in the current casino, the most recently selected machine and whether it led to earning a coin.

\smallskip
\noindent\textbf{Dependent Measures.} As an objective measure, we report the cumulative regret at each casino. We also gave the participants a $5$-point rating scale survey ($1$-Strongly Disagree, $5$-Strongly Agree) consisting of $5$ questions for each algorithm, analogous to the burger stacking experiments: ``I was usually able to select the machine I wanted" (\emph{Ability}), ``The AI agent insisted on some suboptimal machines" (\emph{Insisting}), ``The AI agent was easy to collaborate with" (\emph{Easy}), ``The AI agent was annoying" (\emph{Annoying}), and ``I could earn more coins if I were playing alone" (\emph{Alone}).

%In addition to the survey questions, the participants were asked for comments about the performance of the robot. Replying this question was optional.

\clearpage%\smallskip
\noindent\textbf{Hypotheses.}
\begin{quote}
    \textbf{H3.} \textit{Partner-Aware UCB algorithm will help the users earn more coins, and lead to lower regrets.}\\
    \textbf{H4.} \textit{Users will subjectively perceive the Partner-Aware UCB robot as a better partner.}
\end{quote}

\noindent\textbf{Results-Objective.} We report the cumulative regret for each participant and casino in Figs.~\ref{fig:users} and \ref{fig:casinos}, respectively. The last plot of Fig.~\ref{fig:users} shows the average over both casinos and users.

For $23$ of the $24$ participants, the Partner-Aware UCB helped achieving lower regret with statistical significance for $14$ of the participants (paired-sample $t$-test, $p<.005$ for $4$ users denoted with double asterisks in the figure, and $0.005\leq p<.05$ for $10$ users denoted with a single asterisk). To avoid $p$-hacking, we also performed a two-way repeated measures ANOVA, which again resulted in $p<.005$ between the algorithms.

Similarly, in $21$ of the $25$ casinos, the users incurred lower regret with the Partner-Aware UCB. The regrets are comparable for the remaining $4$ and this can be explained with either the casino being very difficult and thus both algorithms incurring high regrets, or the casino being so easy that both algorithms quickly find the optimal arm with most participants (e.g. the plots with green background). An example of a difficult casino is highlighted with a red background, where both algorithms receive high regrets -- however, Partner-Aware UCB has a sublinear trend and the regret with Na\"ive UCB is increasing linearly, so Partner-Aware UCB could potentially outperform if there were more time steps. Out of the $21$ casinos where Partner-Aware UCB outperformed, the comparison is statistically significant in $16$ casinos ($p\!<\!0.005$ in $8$ and $0.005\!\leq\! p\!<\!0.05$ in the other $8$).

Aligned with the regret values, the Partner-Aware UCB robot also led to higher earnings. While the participants earned $18,\!825.8\pm39.3$ coins over all $25$ casinos with Partner-Aware UCB, this number is only $17,\!717.2\pm235.0$ with Na\"ive UCB. Together with the results on cumulative regret, these strongly support \textbf{H3}.

\begin{figure}[h]
    \centering
    \includegraphics[width=0.5\textwidth]{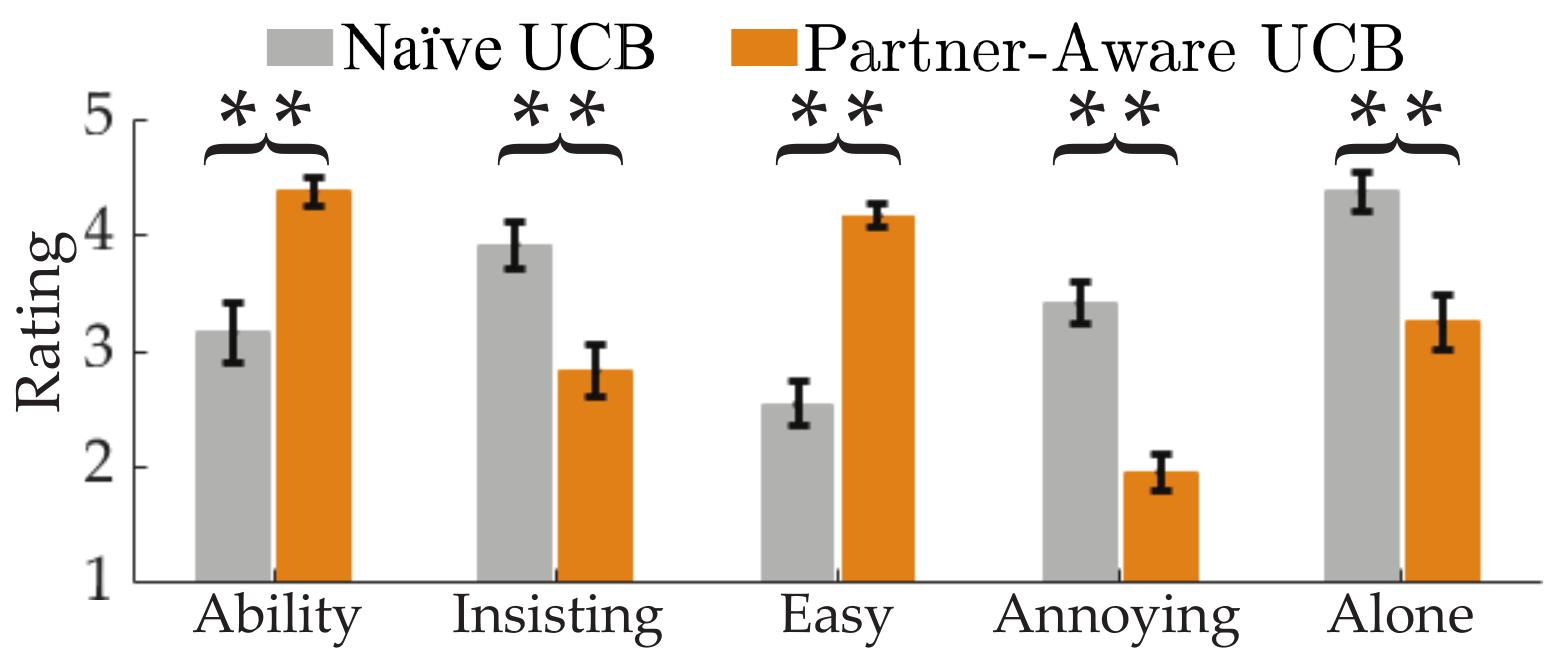}
    \caption{Survey results of the online casino study. Comparisons are significantly in favor of Partner-Aware UCB.}
    \label{fig:survey}
\end{figure}

\noindent\textbf{Results-Subjective.} Our survey results indicate users significantly prefer Partner-Aware UCB: we plot the users' responses to our $5$-point rating scale survey in Fig.~\ref{fig:survey}. We first confirmed the responses were reliable with Cronbach's alpha $ >\!0.85$. The users indicated they were able to select their desired machine (\emph{Ability}) more frequently with the Partner-Aware UCB, and thought the Na\"ive UCB robot was more frequently insisting on suboptimal machines (\emph{Insisting}). Moreover, Partner-Aware UCB was easier to collaborate with (\emph{Easy}), and significantly less annoying (\emph{Annoying}). While the participants, on average, indicated they could earn more coins if they were playing alone\footnote{This is reasonable given that single-agent MAB is easier as it does not require decentralized coordination.}, they were significantly more confident in this after partnered with Na\"ive UCB. All of these subjective results are statistically significant with $p<.005$ and strongly support \textbf{H4}.

\section{Computation Infrastructure}
\label{appendix:infrastructure}
The bandit algorithms in all simulations and the in-lab burger stacking robot experiments have been run on a Lenovo ThinkPad P1 Gen 2 computer with 16 GB RAM and an 8$^\textrm{th}$ Generation Intel$^\text{\textregistered}$ Core$\text{\texttrademark}$ i7-8850H processor (2.60 GHz, up to 4.30 GHz with Turbo Boost, 6 Cores, 12 Threads, 9 MB Cache). Online portion of the burger stacking robot experiments and all of online casino experiments are conducted using Amazon Web Services (AWS) on an Elastic Compute Cloud (EC2) instance with 4 vCPUs and 16 GB RAM.

\end{document}